\documentclass[10pt,twocolumn,letterpaper]{article}

\usepackage[pagenumbers]{cvpr} 
\definecolor{cvprblue}{rgb}{0.21,0.49,0.74}
\usepackage[pagebackref,breaklinks,colorlinks,allcolors=cvprblue]{hyperref}
\usepackage{array}
\newcommand{\rr}[1]{\mathbb{R}^{#1}}
\def\paperID{10329} % *** Enter the Paper ID here
\def\confName{CVPR}
\def\confYear{2025}
\title{Text Embedding is Not All You Need: Attention Control for Text-to-Image Semantic Alignment with Text Self-Attention Maps}

\author{Jeeyung Kim$^*$, Erfan Esmaeili\thanks{Equal contribution.}\;, and Qiang Qiu\\
Purdue University\\
{\tt\small \{jkim17, efakhabi, qqiu\}@purdue.edu}
}
\usepackage{amsthm}
\usepackage{enumitem}
\newcommand{\He}{{H_\text{e}}}
\newcommand{\De}{{D_\text{e}}}
\newcommand{\Hc}{{H_\text{c}}}
\newcommand{\Dc}{{D_\text{c}}}
\newcommand{\Nc}{{N_\text{c}}}

\newcommand{\bos}{\texttt{\footnotesize<bos>\,}}

\newcommand{\ord}[1]{\mathcal{O}(#1)}

\usepackage{amsthm,amsmath,amssymb}
\newtheorem{assumption}{Assumption}
\newtheorem{propx}{Proposition}
\newtheorem{lemmax}{Lemma}

\newenvironment{assumptionx}[1][]{%
  \begin{assumption}[#1]%
}{\hfill$\Box$\end{assumption}}
\newenvironment{prop}[1][]{%
  \begin{propx}[#1]%
}{\hfill$\Box$\end{propx}}
\newenvironment{lemma}[1][]{%
  \begin{lemmax}[#1]%
}{\hfill$\Box$\end{lemmax}}
\usepackage{booktabs}  % For nicer horizontal rules
\usepackage{multirow}  % For merging cells vertically
\usepackage{soul}
\usepackage{pifont}

\newcommand{\xmark}{\ding{55}}

\begin{document}
\maketitle
\begin{abstract}
In text-to-image diffusion models, the cross-attention map of each text token indicates the specific image regions attended. Comparing these maps of syntactically related tokens provides insights into how well the generated image reflects the text prompt. For example, in the prompt, “a black car and a white clock”, the cross-attention maps for “black” and “car” should focus on overlapping regions to depict a black car, while “car” and “clock” should not.
Incorrect overlapping in the maps generally produces generation flaws such as missing objects and incorrect attribute binding. 
Our study makes the key observations investigating this issue in the existing text-to-image models:
(1) the similarity in text embeddings between different tokens---used as conditioning inputs---can cause their cross-attention maps to focus on the same image regions; and (2) text embeddings often fail to faithfully capture syntactic relations already within text attention maps. As a result, such syntactic relationships can be overlooked in cross-attention module, leading to inaccurate image generation. To address this, we propose a method that directly transfers syntactic relations from the text attention maps to the cross-attention module via a test-time optimization.
Our approach leverages this inherent yet unexploited information within text attention maps to enhance image-text semantic alignment across diverse prompts, without relying on external guidance. 
\end{abstract}

\section{Introduction}
\label{sec:intro}
Recent advancements in diffusion models enable generating images based on various text prompts. However, semantic discrepancies often arise between the text and generated images, raising problems such as missing objects---where certain elements are overlooked---and attribute mis-binding---where attributes are incorrectly assigned to subjects. 
\begin{figure}
    \centering
    \includegraphics[width=0.99\linewidth]{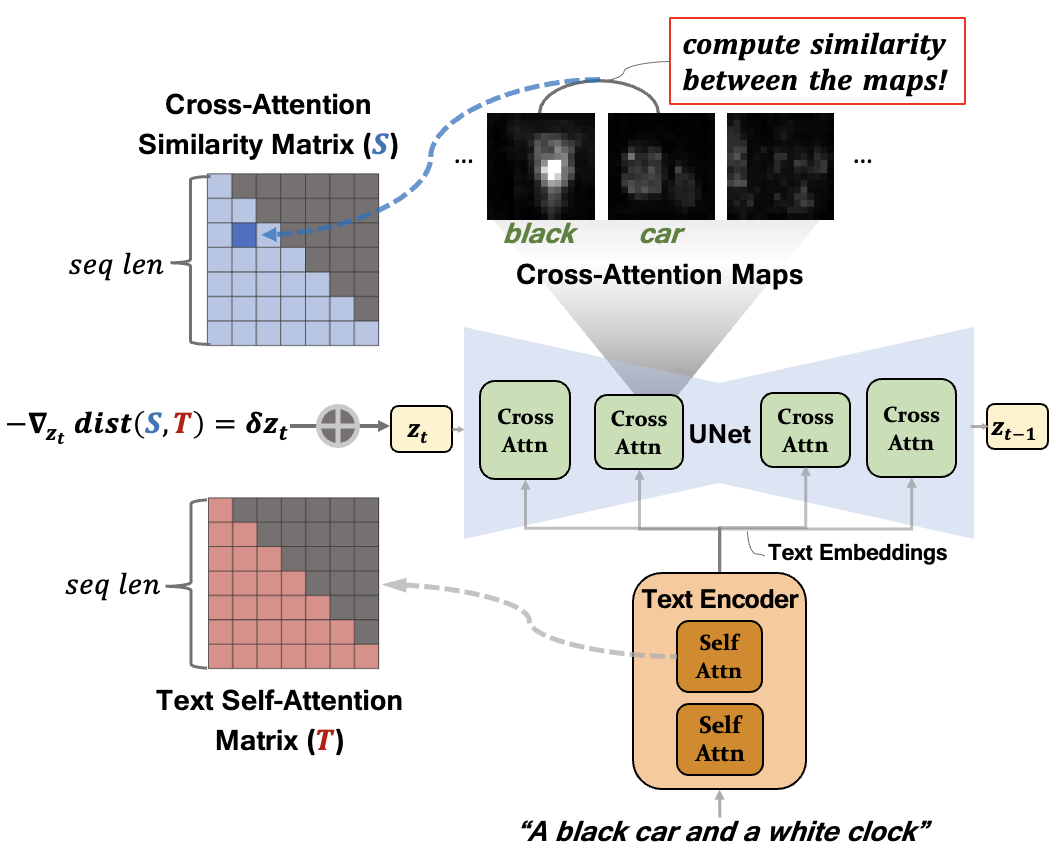}
    \caption{The overview of our method. We leverage text self-attention matrix and optimize the latent noise ($z_t$) by minimizing the distance between the cross-attention similarity matrix ($\mathsf{S}$) and the text self-attention matrix ($\mathsf{T}$). This encourages integrating syntactic relationships into text-to-image diffusion models.}
    \label{fig:method_overview}
    \vspace{-3mm}
\end{figure}

Prior studies~\cite{hertz2022prompt} demonstrated that 
 the cross-attention map of each token in text-to-image (T2I) diffusion models highlights the attended regions in the image and provides clues about the spatial placement of elements corresponding to the tokens. In particular, \cite{meral2024conform, rassin2024linguistic, agarwal2023star} suggest that the spatial alignment in cross-attention maps among related words influences the fidelity of images to the text prompts, as we also demonstrate in Section~\ref{sec:empirical}. For instance, in the prompt \textit{a black car and a white clock}, if the cross-attention maps for \textit{car} and \textit{clock} overlap excessively, unique token contributions can dilute, potentially omitting one object. Conversely, if the maps for \textit{black} (or \textit{white}) and \textit{car} (or \textit{clock}) diverge too much, attribute mis-binding can occur. This implies syntactically related words should ideally have spatially aligned cross-attention maps, as discussed in \cite{rassin2024linguistic, feng2022training}. However, the factors determining this spatial alignment between the maps remain poorly understood.

\begin{figure*}[ht]
    \centering
    \vspace{-2mm}
    \begin{subfigure}[b]{0.27\textwidth}
        \centering
        \includegraphics[width=\textwidth]{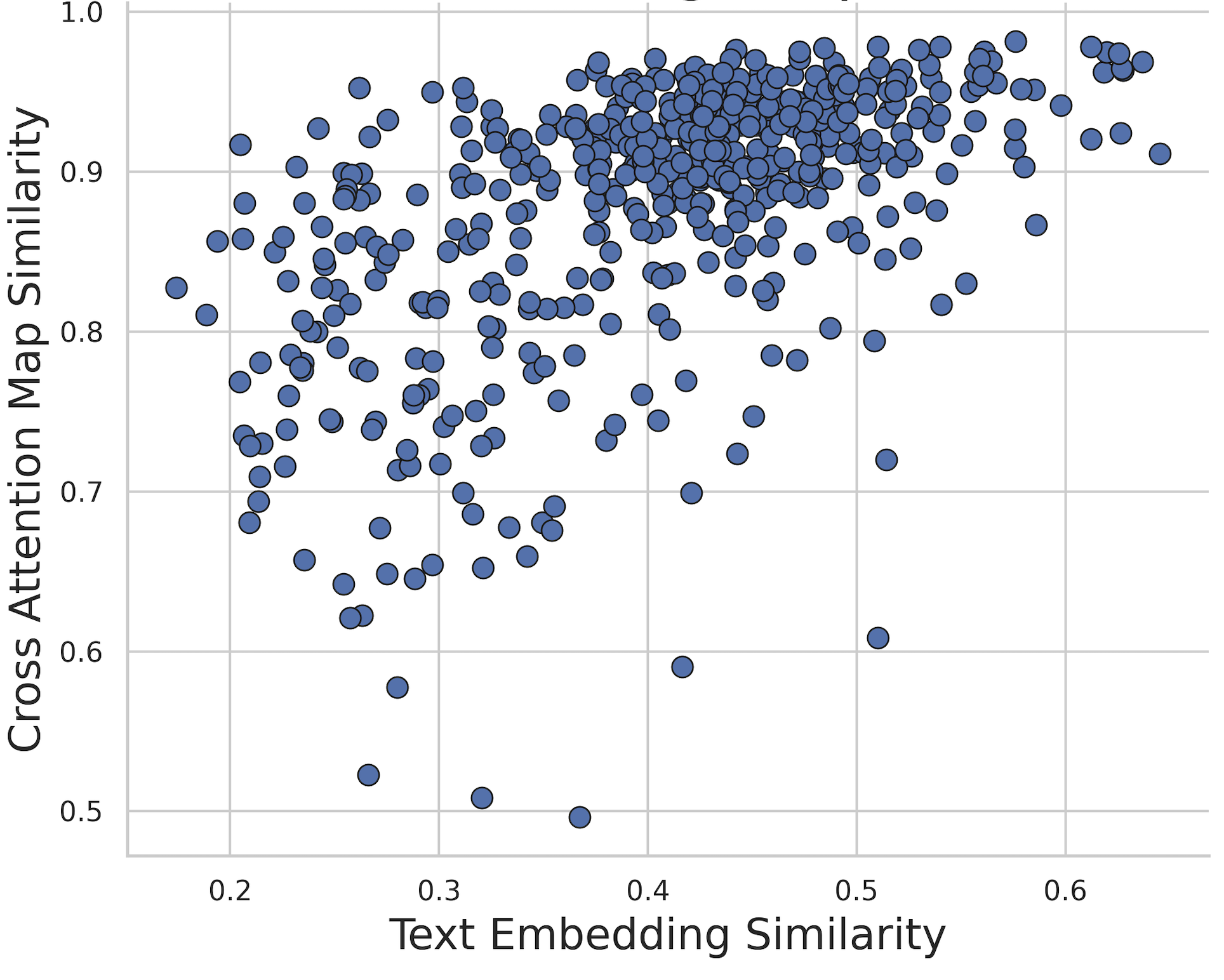} 
        \caption{}
        \label{fig:analysis_sub1}
    \end{subfigure}
    \hfill
    \begin{subfigure}[b]{0.26\textwidth}
        \centering
        \includegraphics[width=\textwidth]{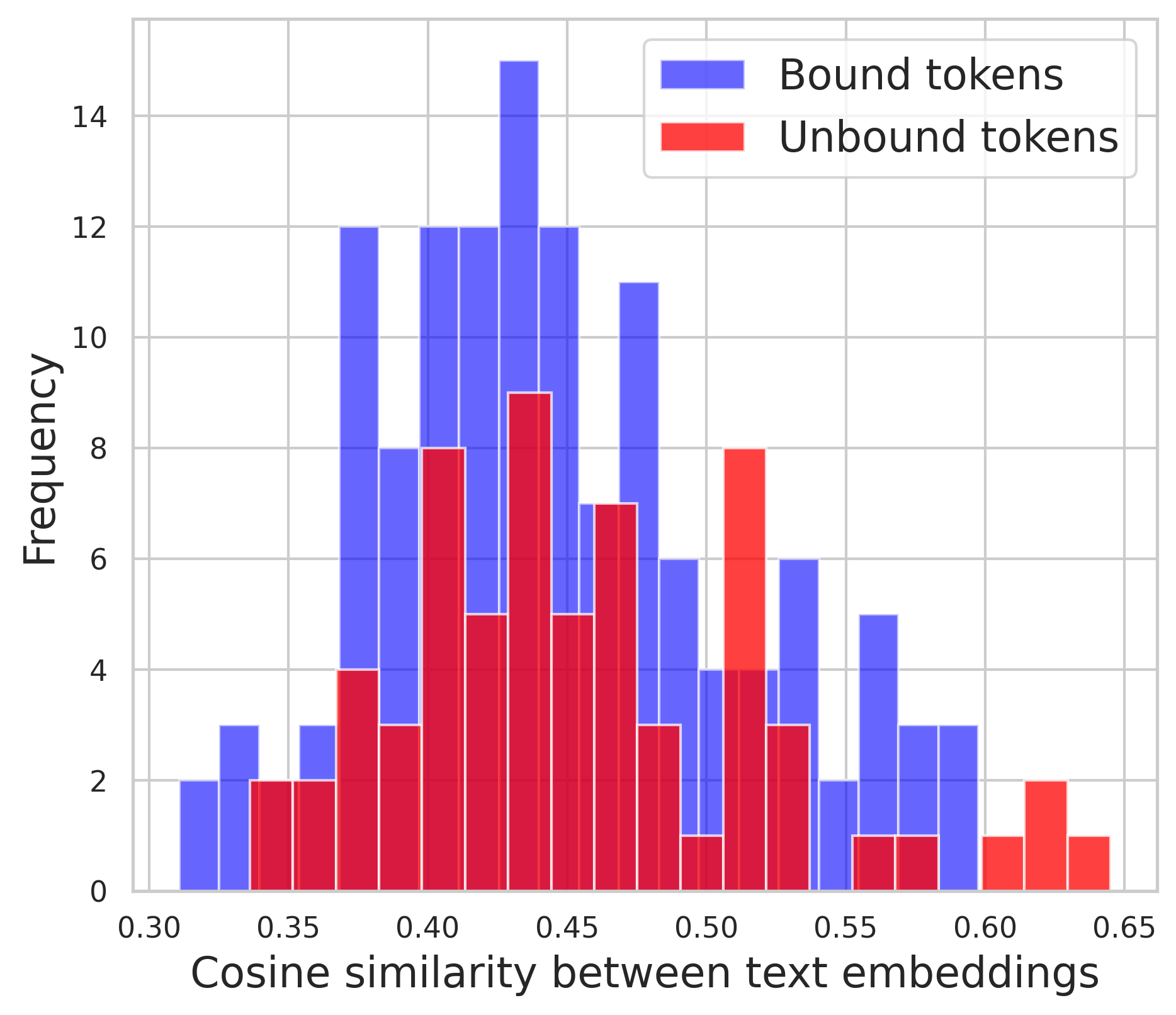} 
        \caption{}
        \label{fig:analysis_sub2}
    \end{subfigure}
    \hfill
    \hspace{1mm}
    \begin{subfigure}[b]{0.45\textwidth}
        \centering
        \includegraphics[width=\textwidth]{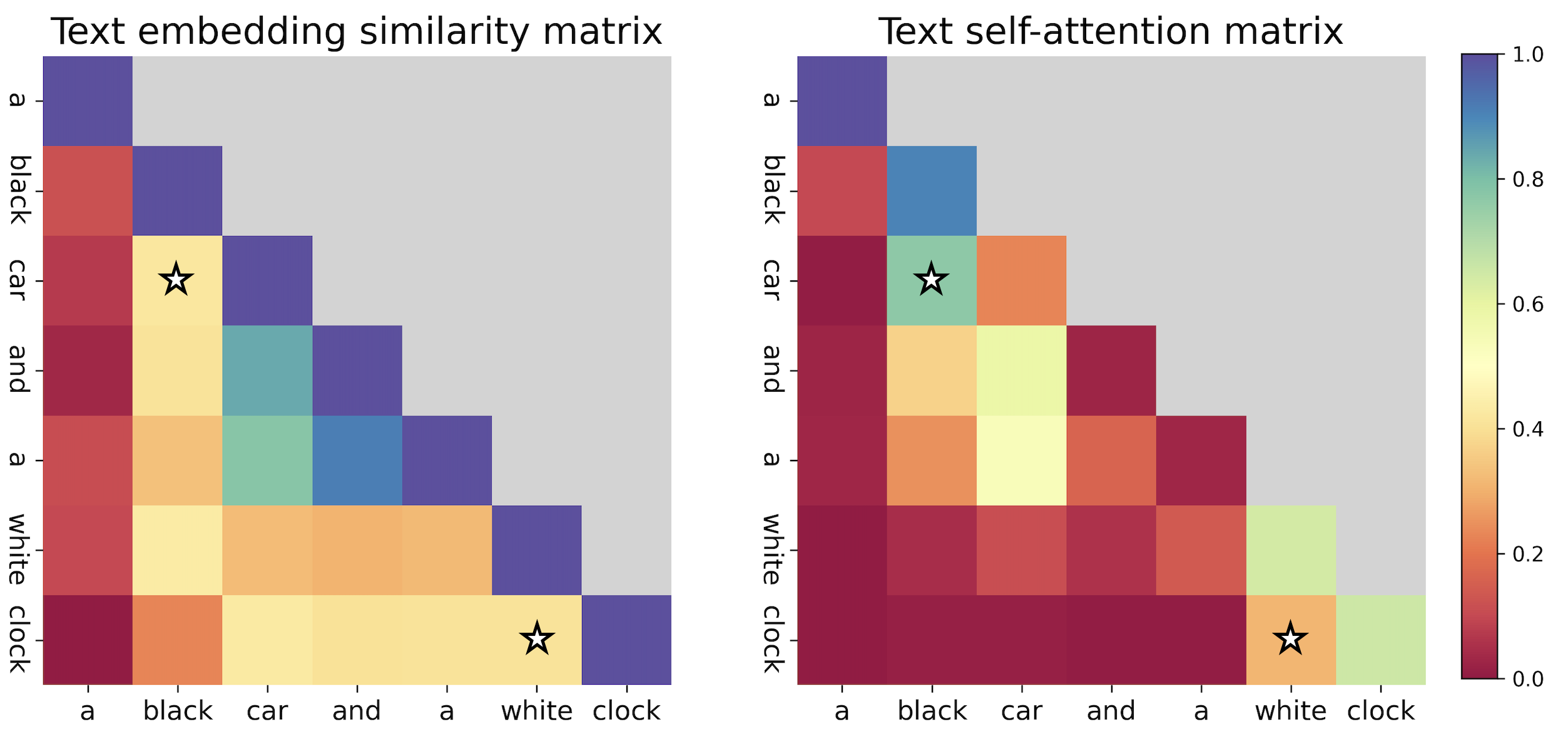}
        \caption{}
        \label{fig:analysis_sub3}
    \end{subfigure}
    \caption{For the analysis, we use the prompt sets from \cite{chefer2023attend}, structured as  ``[$\text{attribute}_1$]\; [$\text{object}_1$] and [$\text{attribute}_2$]\;[$\text{object}_2$]'', ``[$\text{object}_1$] and/with [$\text{object}_2$]'' or ``[$\text{object}_1$]  and [$\text{attribute}_2$]\;[$\text{object}_2$]''. (a) Comparison of the cosine similarity of text embeddings with that of the corresponding cross-attention maps at denoising step 1, with pairs of tokens ($\text{object}_i$, $\text{object}_j$), where $i \neq j$, and pairs of tokens ($\text{attribute}_m, \text{object}_n$) for both $m = n$ and $m \neq n$. As text embeddings become more similar, their cross-attention maps get similar. (b) The distributions of text embedding similarity between i) \textit{Bound tokens}--- ($\text{attribute}_i, \text{object}_i$) for $i=1,2$, and ii) \textit{Unbound tokens}---($\text{object}_1, \text{object}_2$). The distributions show no discernible difference, indicating text embeddings do not effectively represent the syntactic relationships. (c) Comparison of text embedding similarity (left) and the text self-attention map power by 3 (right) for the prompt \textit{a black car and a white clock}. In the self-attention maps ($\mathsf{T}$), \textit{clock} attends more to \textit{white}, unlike the text embeddings.}
    \label{fig:three_subfigures}
    \vspace{-2mm}
\end{figure*}

In our study, we first investigate the factors that contribute to the spatial alignment of cross-attention maps across different tokens, which can ultimately affect the accuracy of T2I generation. We reveal that text embeddings, which function as keys in the cross-attention modules, play a pivotal role in determining the similarity of these maps, as illustrated in Figure~\ref{fig:analysis_sub1}. Specifically, as the text embeddings for different words become more similar, the cosine similarity of their corresponding cross-attention maps increases. This effect is noticeable from the initial denoising step.

Considering that (i) cross-attention maps should capture syntactic relationships between words and that (ii) the spatial alignment of cross-attention maps is influenced by the text embeddings, we question whether text embeddings accurately capture the linguistic structure in text prompts. 
% syntactic relationships between words (e.g., \textit{car}-\textit{clock} or \textit{car}-\textit{black}) in text prompts. 
Our findings suggest they do not, as illustrated in Figure~\ref{fig:analysis_sub2} and the left of Figure~\ref{fig:analysis_sub3}. Syntactically related words in prompts (\eg, \textit{black}-\textit{car} and \textit{white}-\textit{clock}) do not necessarily yield similar text embeddings, as detailed in Section~\ref{sec:empirical}. This finding is further supported by \cite{yuksekgonul2022and}, which argues that CLIP~\cite{radford2021learning} often behaves like a bag-of-words model. As a result, the cross-attention maps, derived from text embeddings, can likely fail to faithfully reflect the syntactic relationships.
In the end, the T2I diffusion models can often generate images that inaccurately represent the prompts. 

To address the above issue, prior studies \cite{rassin2024linguistic, feng2022training, agarwal2023star, meral2024conform} resort to external sources to obtain syntactic information within text prompts and regularize cross-attention maps to incorporate these relationships. Specifically, \cite{rassin2024linguistic, feng2022training} employ external text parsers to obtain linguistic structure, while \cite{agarwal2023star, meral2024conform} rely on human intervention to manually group tokens based on syntactic relationships. However, these methods are limited by their dependence on external inputs.

\emph{Do we really need to rely on external sources to obtain syntactic relations within a sentence?}
Remarkably, T2I diffusion models inherently capture syntactic relationships within prompts through the self-attention maps in their text encoder. In the encoder’s self-attention module, each token pays more attention to related words, leading to encoding the entire sentence effectively. 
As shown in the right of Figure~\ref{fig:analysis_sub3}, the text self-attention maps exhibit higher attention scores between the syntactically related words. 
However, this information is weakly encoded in the text encoder outputs (\ie, text embeddings). We conjecture this is due to the text attention module's strong focus on the \texttt{<bos>} token, known as \textit{attention sink}~\cite{sun2024massive, xiao2023efficient}, which minimizes the influence of other tokens, as further discussed in Section \ref{sec:empirical}.

As illustrated in Figure~\ref{fig:method_overview}, our approach reuses the self-attention maps from the text encoder and directly transfers the syntax information to the diffusion models.
We first define a similarity matrix whose values indicate the similarity between pairs of cross-attention maps.
Then, we update the latent noise in diffusion models to minimize the distance between the cross-attention similarity matrix and the text self-attention matrix during inference. 
This allows for the seamless integration of syntactic relationships into the diffusion process.
By guiding cross-attention modules to better capture contextual relationships, our method ultimately produces images accurately reflecting the intended meaning of prompts.
We leverage overlooked information already embedded in T2I models, offering two key advantages: \textbf{(i) Self-contained}: no need for external inputs like text parsers or manual token indices; and \textbf{(ii) Generalizable}: effective across diverse sentence structures.

\section{Related Work}
 \noindent\textbf{Text-to-image generation.}
Text-to-image (T2I) generation \cite{ramesh2022hierarchical,podell2023sdxl,esser2024scaling,chen2023pixart,chen2024pixart,saharia2022photorealistic} is commonly based on latent diffusion models \cite{rombach2022high} where text data is processed via a text encoder \cite{radford2021learning,raffel2020exploring}. 
An important problem in T2I models is the lack of full correspondence between the input text and the image. Although diverse methods \cite{agarwal2024alignit,feng2022training,hu2024ella,yu2024uncovering,shen2024rethinking,qi2024not,manas2024improving} are proposed to diagnose and fix this problem, a common theme in prior works is adjusting the cross-attention maps in inference time, following the seminal work \cite{hertz2022prompt}.\\ 
 \noindent\textbf{Cross-attention control for improved sematic alignment.} Several defects in cross-attention maps hinder sematic alignment, including attention dominance \cite{zhang2024enhancing}---one token getting huge attention weight, and attention leakage/overlap \cite{yang2023dynamic,marioriyad2024attention}---attention weights not respecting the spatial boundaries of the intended objects. In particular, one line of work employs contrastive objectives to guide attention optimization \cite{agarwal2023star,meral2024conform,rassin2024linguistic,jiang2024mc}. For example, {\textit{\footnotesize CONFORM}} \cite{meral2024conform} uses prompts with distinct groups, where tokens in the same (opposite) group are treated as positive (negative) pairs. Similarly, \cite{rassin2024linguistic} uses text parser models to identify such pairs. Other approaches, like those in \cite{wang2023tokencompose, kim2023dense}, optimize attention maps with spatial guidance (e.g., segmentation maps, masks, or layouts). However, these methods rely on external resources for text-to-image generation.

 \noindent\textbf{Attention sinks.} The attention weights in pretrained transformers tend to be heavily biased towards special tokens such as \texttt{<bos>}, \texttt{<eos>}, and punctuation tokens, as studied in (vision-)language models \cite{xiao2023efficient, gu2024attention, sun2024massive}. The T2I diffusion model literature \cite{chefer2023attend, yi2024towards} also notes the focus of attention scores on the \texttt{<bos>} token in the CLIP text encoder.
\section{Preliminaries}
\label{sec:prelim}
In this section, we briefly review the structure of text-to-image models, Stable Diffusion~\cite{rombach2022high} (SD), including the text encoder and the cross-attention module.

\noindent\textbf{Text encoder.} To condition the text data for the diffusion process, a text encoder is needed. The input text is first tokenized and then converted into dense vectors.
The dense vectors are then processed via a series of multi-head self-attention layers \cite{vaswani2017attention}. 

At a given layer $\ell$ in the text encoder, let us denote the \textit{key} by  $\mathbf{e}^{(\ell)}_i\in \mathbb{R}^{\He\De}, i =1,\cdots, s$, where $s$ is the key sequence length, $\De$ is the embedding dimension per head and $\He$ is the number of heads in the text encoder. The corresponding self-attention matrix $T^{(\ell,h)}\in\mathbb{R}^{s\times s}$ is given by
\begin{align}\label{text sa matrix}
    T^{(\ell,h)}_{ij}=\frac{\exp(\omega_{ij})}{\sum_{k}\exp(\omega_{ik})},\;
    \omega_{ij}:=\mathbf{e}^{(\ell)\top}_iW^{(\ell, h)}_\text{en }\mathbf{e}^{(\ell)}_j,
\end{align}
where $W^{(\ell, h)}_{\text{en}}\in \mathbb{R}^{\He\De\times \He\De}$ is a pretrained matrix at head $h$ in layer $\ell$ of the text encoder.

In our proposed method, we use the self-attention matrix averaged over layers and heads:
\begin{equation}
\label{eq:text_self_map}
    \mathsf{T'} =\frac{1}{L_{\text{e}}\He}\sum_{\ell=1}^{L_{\text{e}}}\sum_{h=1}^{\He}T^{(\ell,h)}.
\end{equation}
Due to the high attention probabilities assigned to \bos and {\texttt{\footnotesize<eos>\,}} tokens, we remove their corresponding values and re-normalize each row as follows:
\begin{equation}
\label{eq:self_attn_renorm}
    \mathsf{T}_{ij} = \frac{\mathsf{T'}_{ij}}{\sum_{m=2}^{i}\mathsf{T'}_{im}}.
\end{equation}
We denote the final output of the text encoder by $\mathbf{k}_i\in\mathbb{R}^{\He\De}$ where $i=1, \cdots, s$, which used as conditioning inputs in T2I diffusion models. We call $\mathbf{k}_i$ as \textit{text embeddings}. 

\noindent\textbf{Denoising latent variables.}
 \emph{Latent diffusion models} are a class of generative models that generate latent tensors ($z_0$) of an image. 
A latent diffusion model $D_{\theta}$ learns to simulate the \emph{denoising} process: starting from Gaussian noise $z_{\tau}$, the model iteratively reconstruct $z_0$ by predicting the noise at step $t$:
    \begin{equation}
        z_{t-1} =z_t- D_{\theta}(z_t; \{\mathbf{k}_i\}),\quad 0< t\leq \tau.
    \end{equation}

\noindent\textbf{Cross-attention module.} 
In T2I latent diffusion models, the interaction between text and image data is performed by multi-head cross-attention. 
In the cross-attention layer $\ell$, we define the \emph{query}, $\mathbf{q}^{(\ell)}_a\in\mathbb{R}^{\Hc\Dc }$, where $a=1,\cdots,\Nc$, and $\Nc$ is the query sequence length at the given cross-attention layer. $\Dc$ is the hidden dimension per head, and $\Hc$ is the number of heads of the cross-attention layer. 

The cross-attention map $A^{(\ell,h)}\in \mathbb{R}^{\Nc\times s}$ at  head $h$ with elements $A^{(\ell,h)}_{ai}$ is defined as
\begin{align}
    A_{ai}^{(\ell,h)} := \frac{\exp(\Omega_{ai})}{\sum_{j=1}^s\exp(\Omega_{aj})},\;
    \Omega_{ai}:=\mathbf{q}_a^{(\ell)\top} W^{(\ell, h)}_\text{c} \mathbf{k}^{(\ell)}_i.
\label{eq:cross_attn}
\end{align}
Here, $W^{(\ell, h)}_\text{c}\in\mathbb{R}^{\Hc\Dc\times \Hc\Dc}$ is a  pretrained matrix at head $h$ in layer $\ell$.

\noindent\textbf{Cross-attention similarity matrix.}
In our study, we use the cross-attention maps averaged over heads and over all layers for which $\Nc=M$:
\begin{equation}
\label{eq:attention_map}
    \mathsf{A}=\frac{1}{L_{M}\Hc}\sum_{\ell=1}^{L_{M}}\sum_{h=1}^{\Hc}A^{(\ell,h)}.
\end{equation}
Here, $L_{M}$ is the number of cross-attention layers with $\Nc=M$. Based on the average cross-attention map above, we define the similarity matrix $\mathsf{S}\in\mathbb{R}^{s\times s}$ as follows:
    \begin{align}\label{cos}
      \mathsf{S}_{ij}\hspace{-0.5mm}:=\hspace{-0.5mm}\frac{\mathsf{C}_{ij}}{\sum_{k=1}^s \mathsf{C}_{ik}}, \mathsf{C}_{ij} \hspace{-0.5mm}:=\hspace{-0.5mm}\frac{\sum^{\Nc}_{a=1} \mathsf{A}_{ai}\mathsf{A}_{aj}}{\big(\sum^{\Nc}_{a=1}\mathsf{A}^2_{ai}\big)^{\frac12}\big(\sum^{\Nc}_{a=1}\mathsf{A}^2_{aj}\big)^{\frac12}}.
    \end{align}
$\mathsf{C}_{ij}$ denotes cosine similarity between maps corresponding to $i$ and $j$.

\section{Understanding and Resolving Text-to-Image Semantic Discrepancy}
\label{sec:evidence}
In this section, we examine the cross-attention maps to identify factors that lead to semantic discrepancies between generated images and text prompts. Based on these findings, we introduce an approach to enhance the fidelity of T2I diffusion models by leveraging text self-attention maps to regularize cross-attention maps.

\subsection{Why Do Generated Images Misrepresent Text?}
\label{sec:empirical}
Our findings are based on the premise that cross-attention maps should capture syntactic relationships between words for accurate T2I generation, as supported by prior works \cite{meral2024conform, agarwal2023star, rassin2024linguistic, feng2022training}. We identify a key factor contributing to incorrect relations in cross-attention maps, resulting in less faithful image generation: While text embedding ($\mathbf{k}_i$)  similarity strongly correlates with the cross-attention similarity matrix $\mathsf{C}$, it insufficiently reflects syntactic bindings in the text prompt. This highlights a fundamental issue in text-image semantic discrepancy: \emph{the similarity in text embeddings lack the syntax information necessary for accurate image generation}. We reveal that text self-attention maps ($\mathsf{T}$) effectively capture syntactic information, but this is not sufficiently encoded in text embeddings. The absence of the information in text embeddings can be due to an artifact in the attention module, known as  \textit{attention sink} \cite{xiao2023efficient}, where attention weights are biased toward the \bos token as detailed later. These findings motivate the solution presented in the following section.

For the following empirical analysis, we use the prompt sets introduced in \cite{chefer2023attend}, containing prompts structured as the following  three categories:

\noindent (i) $[\texttt{\footnotesize attribute}_1][\texttt{\footnotesize object}_1]$ \emph{and} $[\texttt{\footnotesize attribute}_2][\texttt{\footnotesize object}_2]$,

\noindent (ii) 
 $[\texttt{\footnotesize object}_1 (\texttt{\footnotesize animal})]$ \emph{with} $[\texttt{\footnotesize object}_2]$,
 
\noindent (iii)
 $[\texttt{\footnotesize object}_1(\texttt{\footnotesize animal})]$ \emph{ and} $[\texttt{\footnotesize attribute}_2][\texttt{\footnotesize object}_2]$.\\
 Prompt sets in the above formats enable us to focus on two key cases of text-image semantic discrepancies: missing objects and attribute mis-binding. In the following, we refer to the group of \textit{syntactically bound words (tokens)} as pairs of $([\texttt{\footnotesize attribute}_i], [\texttt{\footnotesize object}_i])$, $i=1,2$, and the group of \textit{syntactically unbound words (tokens)} as $([\texttt{\footnotesize attribute}_i], [\texttt{\footnotesize object}_j])$ and $([\texttt{\footnotesize object}_1], [\texttt{\footnotesize object}_2])$, where $i \neq j$.

% \noindent\textbf{1. The (dis-)similarity between cross-attention maps for different tokens contributes to text-image misalignment.}
 
% \noindent\textbf{1. Cross-attention maps should reflect the syntactic relationships in input texts to generate accurate images.}
We revisit the role of cross-attention maps in text-image semantic coherence, where spatial overlap or separation is crucial, as discussed in \cite{meral2024conform, agarwal2023star, rassin2024linguistic, feng2022training}.  Figure~\ref{fig:cross_attn_effect} illustrates this effect: On the left, overlapping attention maps for syntactically unbound words lead to object missing, while greater overlap for syntactically bound words enhances attribute binding. This suggests that syntactic associations should be reflected in cross-attention maps.

\begin{figure}
    \centering
    \includegraphics[width=1.02\linewidth]{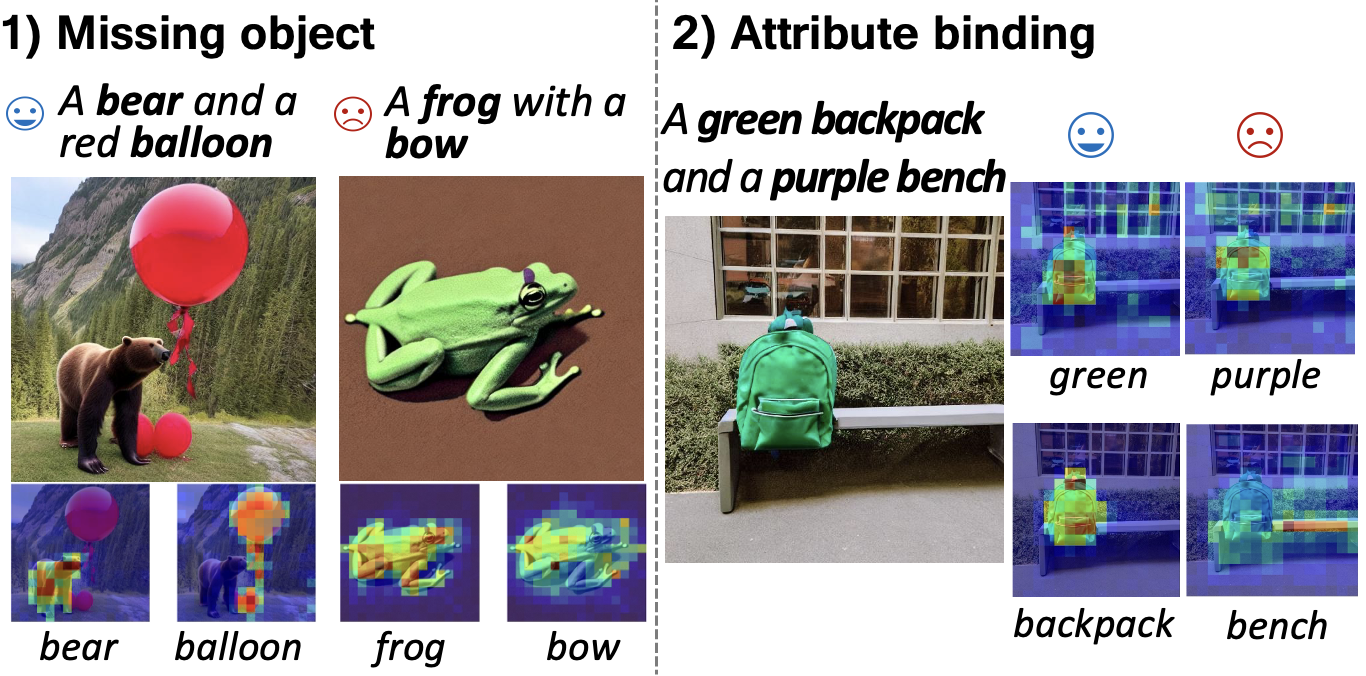}
    \caption{The generated images and cross-attention maps ($\mathsf{A}$) for the specific tokens from SD v1.5. This illustrate the importance of spatial alignment in cross-attention maps for accurate image generation. Divergent (overlapping) cross-attention maps for syntactically unbound (bound) words enhances text-to-image fidelity.}
    \label{fig:cross_attn_effect}
\end{figure}
\begin{figure}[tt]
    \centering
    \begin{subfigure}[b]{0.23\textwidth}
        \centering
        \includegraphics[width=\textwidth]{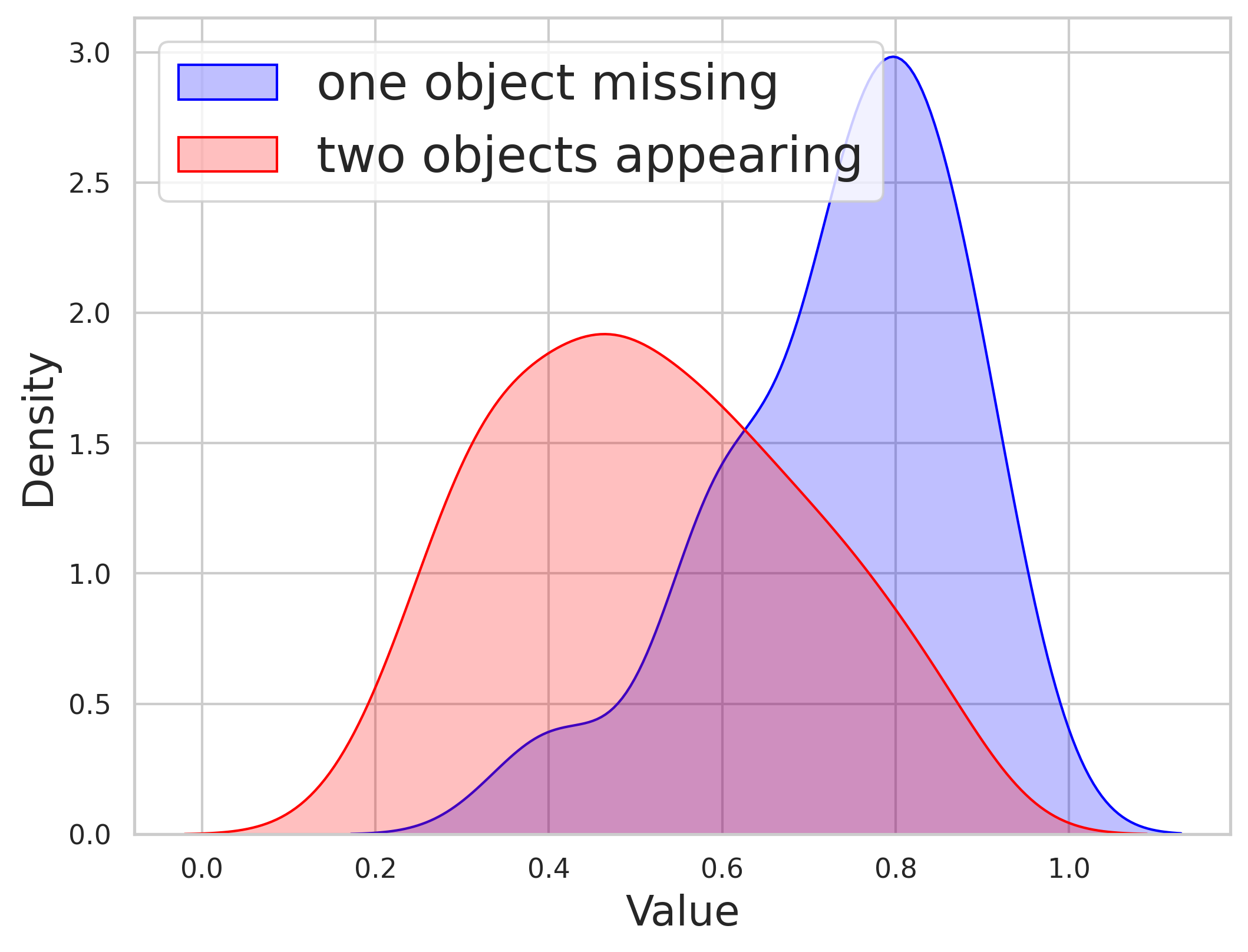}
        \caption{}
        \label{fig:missing_objects_tifa}
    \end{subfigure}
     % \vspace{1em}
     % \hfill
    \begin{subfigure}[b]{0.23\textwidth}
        \centering
        \includegraphics[width=\textwidth]{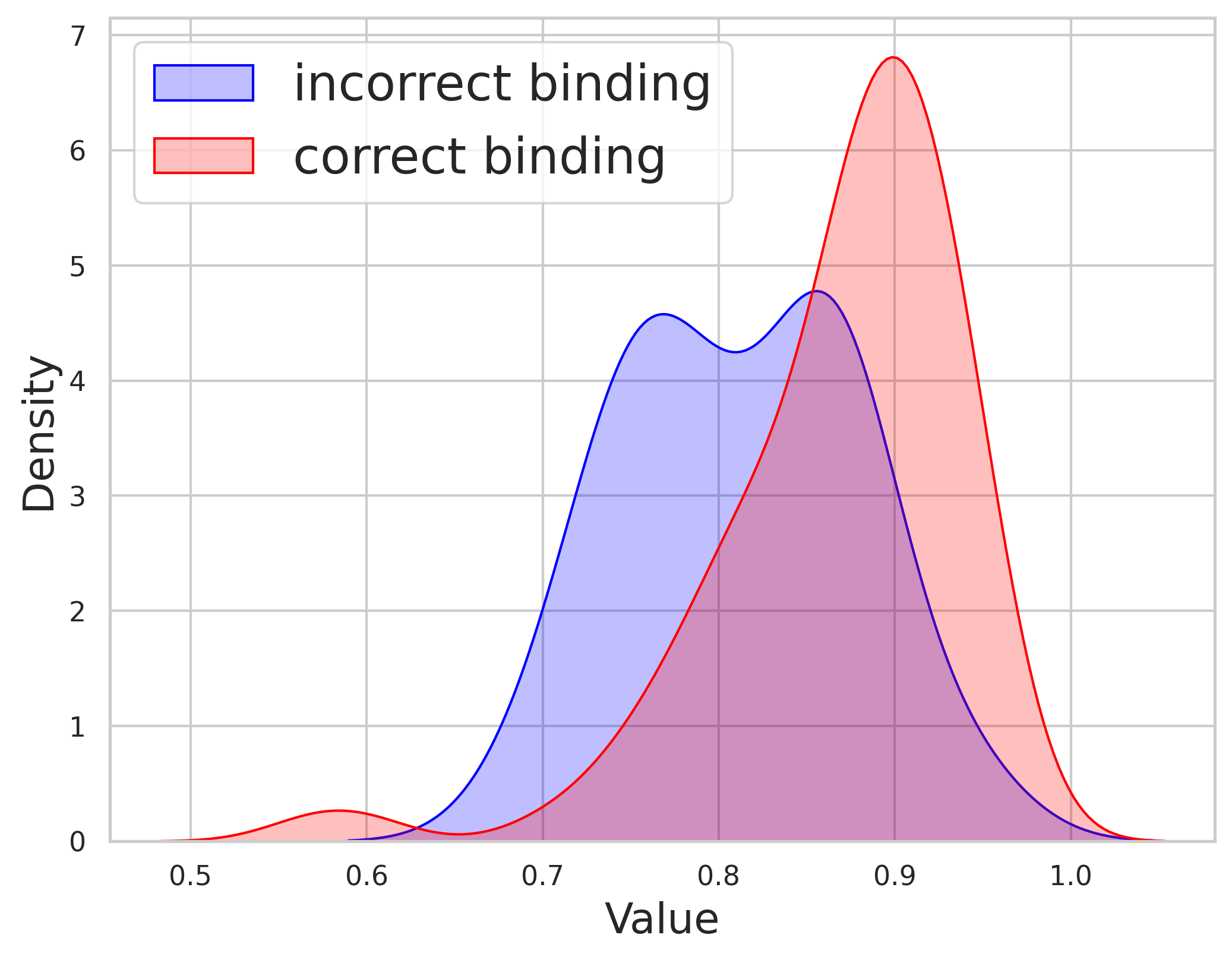}
        \caption{}
        \label{fig:attr_binding_tifa}
    \end{subfigure}
    % \vspace{-2mm}
\caption{Comparison for the distributions of cosine similarity between cross-attention maps (at denoising step 10). (a) The cases with one missing object--incorrect and two objects present--correct. (b) The cases with incorrect and correct attribute binding. Correct instances are more frequent when the cosine similarity is low for \textit{objects presence} and high for \textit{attribute binding}.}
\label{fig:tifa}
% \vspace{-2mm}
\end{figure}
To statistically assess cross-attention map's impact, we analyze categories (ii) and (iii) of the aforementioned prompt sets. Assuming spatial overlap is measurable by the cosine similarity matrix ($\mathsf{C}$) in eq.\eqref{cos}, we compute cosine similarities for the attention maps of $\textit{object}_1$ and $\textit{object}_2$ in missing objects, and for $\textit{attribute}_2$ and $\textit{object}_2$ in attribute binding cases. Then, we compare cosine similarity distributions for correct vs. incorrect images on both cases. See Appendix B for setup details. Figure~\ref{fig:tifa} shows that lower cosine similarity tends to correlate with object presence, while higher similarity supports more accurate bindings.

This observation motivates us to examine the factor contributing to similarity in cross-attention maps across tokens.

\noindent\textbf{Finding 1: The cosine similarity of text embeddings has a large correlation with cross-attention similarity matrix $\mathsf{C}$.}
\begin{figure}
% \vspace{-1mm}
    \centering
    \includegraphics[width=1.0\linewidth]{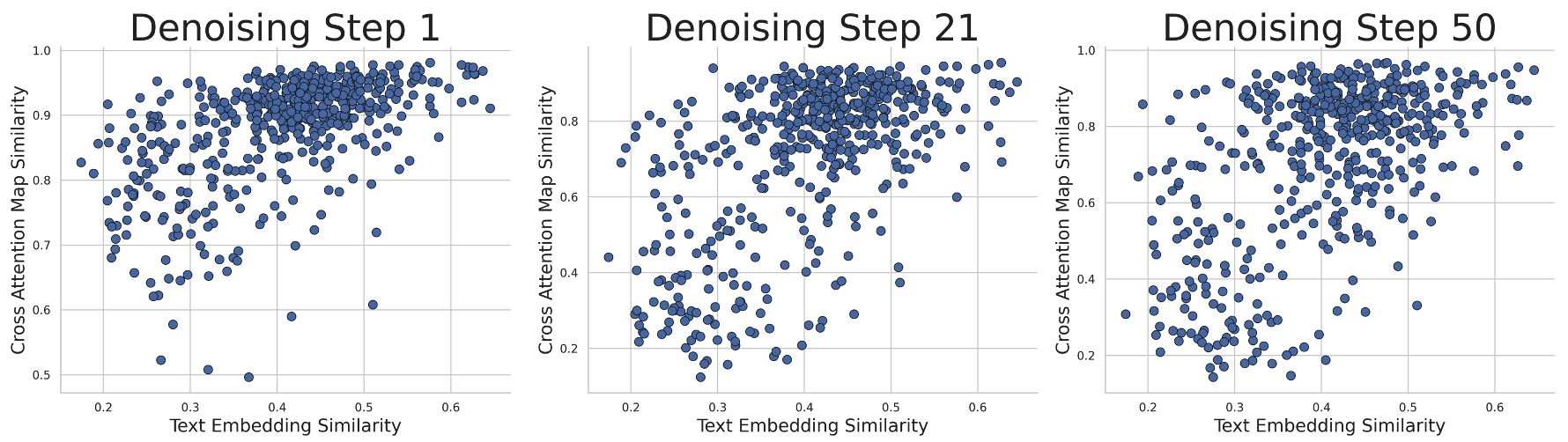}
    \caption{Correlation between the cosine similarity of text embeddings and that of cross-attention maps across denoising steps ($t=1, 21, 50$). Similar text embeddings generally lead to similar cross-attention maps, with the correlation weakening over time.}
    \label{fig:cross_attn_over_time}
    % \vspace{-3mm}
\end{figure}
Figure~\ref{fig:cross_attn_over_time} shows there is a correlation between the cosine similarity of text embedding and $\mathsf{C}$, which persists throughout the final denoising steps. This indicates that similar text embeddings can result in overlapping cross-attention maps.

Next, we justify this finding mathematically.
\begin{prop}
    If $A^{(\ell,h)}\in\mathbb{R}^{ \Nc\times s}$ is a cross-attention map defined in eq. \eqref{eq:cross_attn}, then under the assumptions i, ii, 
 and iii described in Appendix A, the cosine similarity matrix can be written in terms of key vectors $\mathbf{k}_i^{(\ell,h)}\in \mathbb{R}^{\Hc\Dc}$ as
\begin{align}
    \cos&(A_i^{(\ell,h)},A_j^{(\ell,h)})= \nonumber\\
   &\;\; \exp\Big(-\frac{1}{2}(\mathbf{k}_i-\mathbf{k}_j)^\top  W^2(\mathbf{k}_i-\mathbf{k}_j)\Big),
\end{align}
up to terms of at least $\mathcal{O}(1/\sqrt{\Nc})$ and $\mathcal{O}(\epsilon)$, where $W^2 :=W^{(\ell,h)\top}_\text{c} \Sigma^{(\ell)} W^{(\ell,h)}_\text{c}$ and $\Sigma^{(\ell)}\in\mathbb{R}^{\Hc\Dc \times \Hc\Dc}$is the covariance matrix of query vectors.
\end{prop}
\noindent Refer to Appendix A for the proof.

We empirically and mathematically showed that the similarity in text embeddings $\mathbf{k}_i$ influences cross-attention maps. Next, we evaluate whether the similarity of these embeddings reflects syntax information in text inputs.

\noindent\textbf{Finding 2-1: There is no significant correlation between word syntactic bindings and text embedding similarity.}
Prior study~\cite{yuksekgonul2022and} suggests CLIP embeddings, used in SD, behave like a bag-of-words model, ignoring word relationships.
We also show in Figure~\ref{fig:analysis_sub2} the similarity in CLIP text embeddings does not correlate with syntactic bindings. Specifically, we expect close embeddings for \textit{syntactically bound tokens} and distant embeddings for \textit{unbound tokens}, yet the distributions lack separation.

\noindent\textbf{Finding 2-2: Text self-attention maps do have syntax information.}
We examine the text encoder that produces text embeddings. Interestingly, the self-attention maps in the encoder capture syntactic relationships, as shown in Figure~\ref{fig:text_self_attn}. These maps reveal higher similarity between \textit{syntactically bound tokens} and lower similarity between \textit{unbound tokens}. For more complex prompt structures, see the Appendix B. 
\begin{figure}
    \centering
    \begin{subfigure}[b]{0.28\textwidth}
        \centering
        \includegraphics[width=\textwidth]{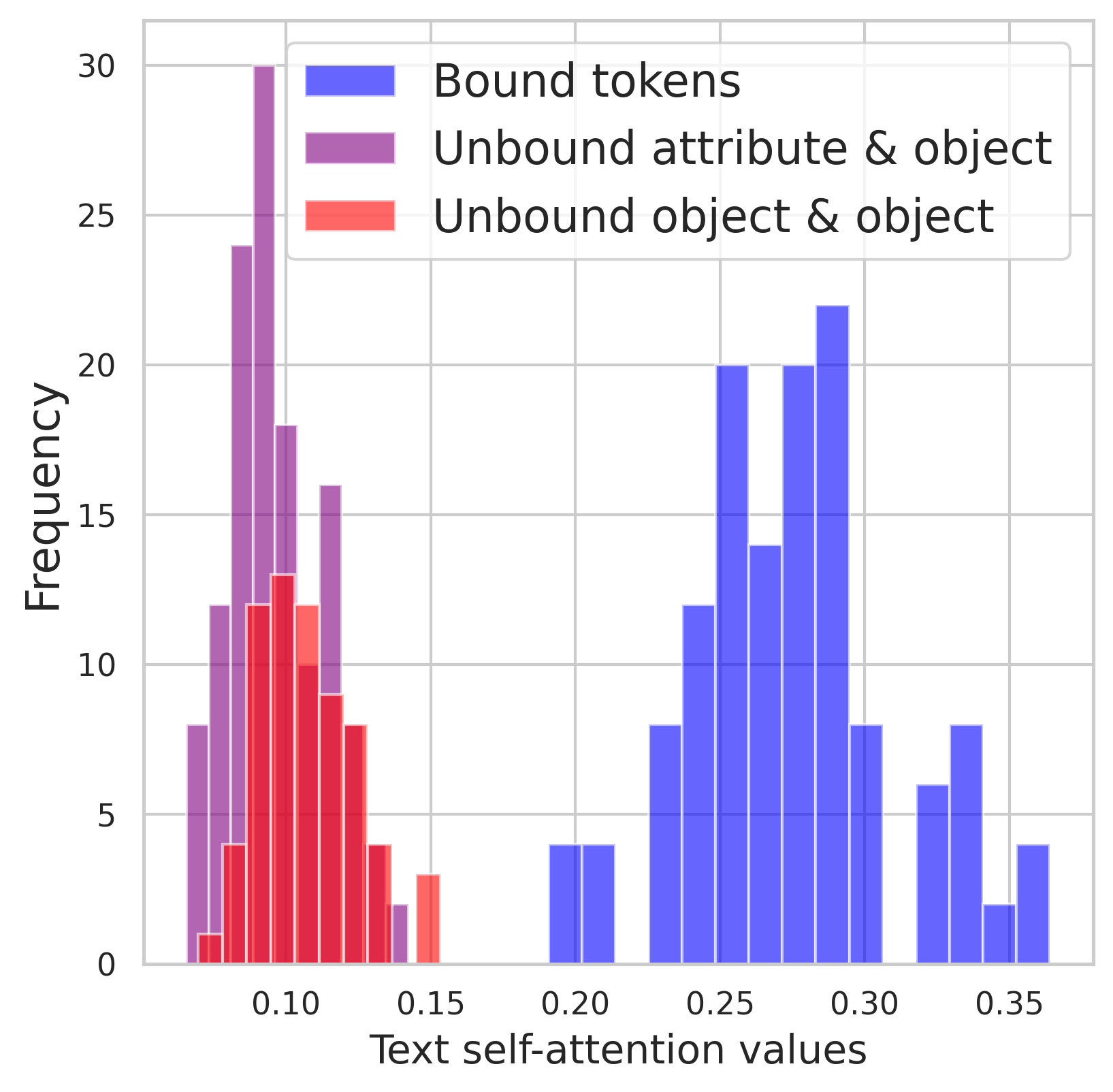}
        \caption{}
        \label{fig:text_self_attn}
    \end{subfigure}
    \begin{subfigure}[b]{0.18\textwidth}
        \centering
        \includegraphics[width=\textwidth]{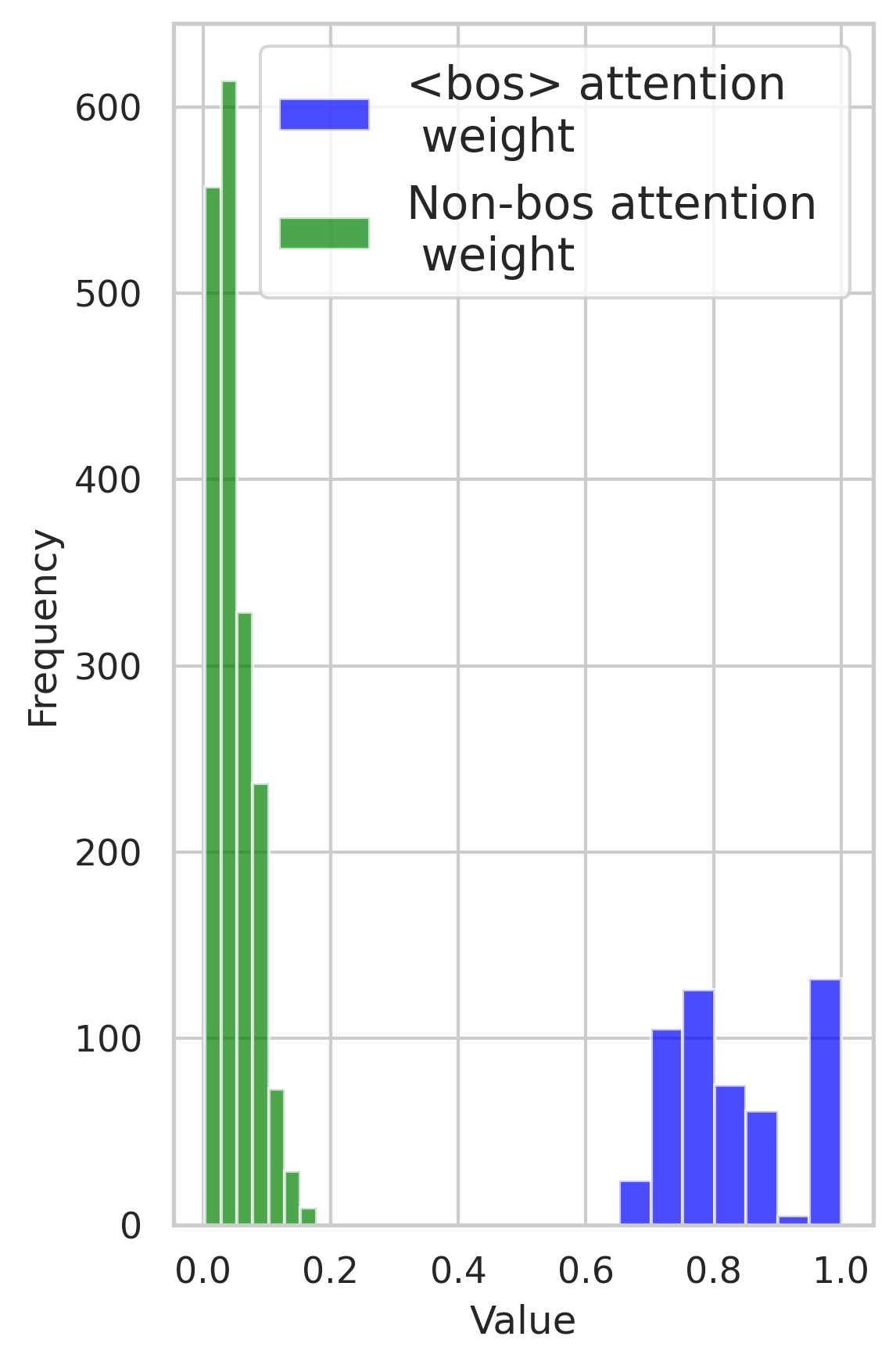}
        \caption{}
        \label{fig:attn_sink}
        \end{subfigure}
\caption{(a) The distributions of  text self-attention value ($\mathsf{T}'_{ij}$ in eq.\eqref{eq:text_self_map}) for bound tokens ($\text{attribute}_m$, $\text{object}_m$) and unbound tokens ($\text{attribute}_m$, $\text{object}_n$) / ($\text{object}_1$, $\text{object}_2$), where $m, n \in \{1,2\}$. The separate  distributions indicate text self-attention maps can indeed represent the syntactic relationships. (b) Comparison of text self-attention probability  histograms between \bos token and non-\bos tokens on 100 prompts: The probability allocated to \bos is on average \>20 times larger than that of other tokens. }
% \vspace{-2mm}
\end{figure}

\noindent\textbf{Finding 2-3: The \textit{attention sink} can contribute to why text embeddings lack the syntax information.}
\emph{Why do text embeddings lack relational information despite being derived from multiple self-attention modules?}
We attribute this gap to \textit{attention sink}~\cite{sun2024massive, xiao2023efficient}, where attention scores are concentrated on a few tokens. In CLIP's text encoder, attention is mainly focused on the \bos token, as discussed in \cite{chefer2023attend, yi2024towards} and shown in Figure~\ref{fig:attn_sink}. We hypothesize the focus on the \bos token can limit the transfer of relational information from self-attention maps ($T$) to text embeddings, as attention scores for other tokens remain much smaller, minimizing their influence in each self-attention layer.

To provide a mathematical justification, we express the difference between the output and input of one block of the self-attention module for a token $e_i$ as: 
\begin{align} 
\mathbf{o}^{(\ell,h)}_i &= \sum_{j=1}^i T^{(\ell,h)}_{ij}W_{\text{v}}^{(\ell,h)} \mathbf{e}^{(\ell)}_j,
\end{align} 
where $W_{\text{v}}^{(\ell,h)}\in\mathbb{R}^{\De\times\He\De}$ is a parameter, $ \mathbf{e}_i \in \mathbb{R}^{\He\De}$, and $T_{ij}^{(\ell,h)}$ is the self-attention matrix in eq. \eqref{text sa matrix}. We consider the situation where \emph{attention sink} occurs, that is the attention weights for the \bos token is much higher than the rest of the sequence:
\begin{equation}
    \epsilon = \frac{\sum^i_{j\neq 1}T_{ij}}{T_{i1}}\ll 1, \qquad i=2,\cdots,s.
\end{equation}
In Appendix A, we prove the following statement:
\begin{prop}
    Define matrix $R\in\mathbb{R}^{s\times s}$ as
    \begin{equation}
        R_{ij}=  \mathbf{e}^{(\ell)\top}_iW_{\text{v}}^{(\ell,h)\top} W_{\text{v}}^{(\ell,h)} \mathbf{e}^{(\ell)}_j,
    \end{equation}
and suppose it has the property
\begin{equation}
    \frac{|R_{mn}|}{R_{11}}\sim\mathcal{O}(1/\epsilon),\qquad \frac{|R_{1m}|}{R_{11}}\sim\mathcal{O}(1),\quad 1<m,n\leq s,
\end{equation}
where $\mathbf{e}^{(\ell)}_1$ is the bos embedding. Then the following holds:
\begin{equation}
    \cos(\mathbf{o}^{(\ell,h)}_i,\mathbf{o}^{(\ell,h)}_j)=1-\mathcal{O}(\epsilon).
\end{equation}
\end{prop}

This suggests that cosine similarity between token vectors ($\mathbf{o}_i^{(\ell, h)}$)---potentially influencing cross-attention maps---remains barely changed across text self-attention layers due to attention sink on the \bos token. In other words, the attention sink can hinder the accurate encoding of self-attention maps into embeddings.

Our findings reveal text embedding alone are insufficient for generating semantically aligned images. \emph{On the other hand, we notably show the potential of transferring neglected syntactic information from text self-attention maps to the cross-attention to enhance T2I semantic alignment.}

\subsection{Text Self-Attention Maps (T-SAM) Guidance }

In the previous section, we show text self-attention maps capture syntax information within a sentence. Building on this insight, we propose leveraging the self-attention maps within the text encoder—a component of diffusion models—to enhance cross-attention maps. By minimizing the distance between the similarity matrix of the cross-attention maps and the text self-attention matrix, our approach ensures that embedded syntactic relationships are effectively transferred to cross-attention.

Our method optimizes cross-attention maps during inference, adjusting their similarity matrix to align with the text self-attention matrix $\mathsf{T}$. The normalized cosine similarity matrix, $\mathsf{S}$ (defined in eq.\eqref{cos}), is used as the cross-attention similarity matrix. This is achieved by simply minimizing the loss function:
\begin{equation}
        \mathcal{L}(z_t)=\sum_{i=1,j\leq i}^s  \rho_i|\mathsf{T}_{ij}^\gamma- \mathsf{S}_{ij}(z_t)|,
\label{eq:loss_fn}
\end{equation}
where the exponent $\gamma$ acts to amplify larger values and compress smaller ones so the effect of controlling temperature and 
$\rho_i=i/s$.
For example, if two words in the prompt have negligible syntactic relation according to the text self-attention matrix, i.e. $\mathsf{T}_{ij}\approx0$, we demand that their similarity of cross-attention maps must \emph{not} be similar: $\mathsf{S}_{ij}\approx 0$.

In practice, this optimization will be applied only to $z_t$ at a few denoising steps during inference as followed:
\begin{equation}
        z'_t = z_t - \alpha \cdot\nabla_{z_t}  \mathcal{L}(z_t).
\label{eq:latent_update}
\end{equation}

\section{Experiments}
\label{sec:exp}

\paragraph{Prompt datasets.}
We evaluate our approach on diverse text prompts using two existing benchmarks.
First, the \textit{TIFA v1.0 benchmark}~\cite{hu2023tifa} is a large-scale text-to-image generation dataset featuring a wide range of sentence structures. This benchmark comprises 4,000 prompts, including 2,000 image captions from the COCO validation set~\cite{lin2014microsoft}, 161 prompts from DrawBench~\cite{saharia2022photorealistic}, 1,420 prompts from PartiPrompt used in Parti~\cite{yu2022scaling}, and 500 texts from PaintSkill~\cite{cho2023dall}.
Second, we use structured prompt sets containing multiple objects and their corresponding attributes from \textit{Attend-n-Excite}~\cite{chefer2023attend}. Prompts in this dataset are grouped into three categories: \textit{Objects} (e.g., “[$\text{attribute}_1$] [$\text{object}_1$] and [$\text{attribute}_2$] [$\text{object}_2$]”), \textit{Animals-Objects} (e.g., “[$\text{animal}$] with [$\text{object}$]” or “[$\text{animal}$] and/with [$\text{attribute}$] [$\text{object}$]”), and \textit{Animals} (e.g., “[$\text{animal}_1$] and [$\text{animal}_2$]”). We exclude the \textit{Animals} category, as it lacks the complex syntactic relations as discussed in \cite{rassin2024linguistic}.
\begin{table}[t]
    \centering
    \vspace{-1mm}
    \caption{Evaluation results for the TIFA benchmark, including TIFA scores and CLIP similarity scores. \textit{External Info.} indicates whether external information is used. CLIP-I (CLIP-T) refers to image-text (text-text) CLIP similarity.}
    \begin{tabular}{l|c|ccc}
    \toprule
         & External Info. &  TIFA & CLIP-I & CLIP-T   \\
         \midrule
         SD & \xmark& 0.79&0.33  & \textbf{0.77}   \\
        LB & \checkmark& 0.80  &0.33 & 0.76  \\

        \textbf{T-SAM} &\xmark&  \textbf{0.83} & \textbf{0.34} & \textbf{0.77}  \\
    \bottomrule
    \end{tabular}
    \vspace{-2mm}
    \label{tab:tifa_scores}
\end{table}

\noindent\textbf{Implementation details.}
Our method (T-SAM) is based on Stable Diffusion (SD) v1.5 \cite{rombach2022high}. We use 50 sampling iterations, updating $z_t$ at each denoising step from 1 to 25. We set $M=256$ in eq.\eqref{eq:attention_map} and $\gamma = 4$ in eq.\eqref{eq:loss_fn}. For the TIFA benchmark, we generate one image per prompt with a shared seed across methods, while using a unique seed for each prompt. We set $\alpha = 40$ in eq.\eqref{eq:latent_update}. For \textit{Attend-n-Excite} prompts, we generate results using 64 seeds, in line with standard practice, setting $\alpha=10$ and applying 20 iterative updates at $i \in \{0, 10, 20\}$. Additional details are provided in the Appendix B.
\begin{figure*}[ht]
    \centering
    \vspace{-2mm}
    \includegraphics[width=0.9\linewidth]{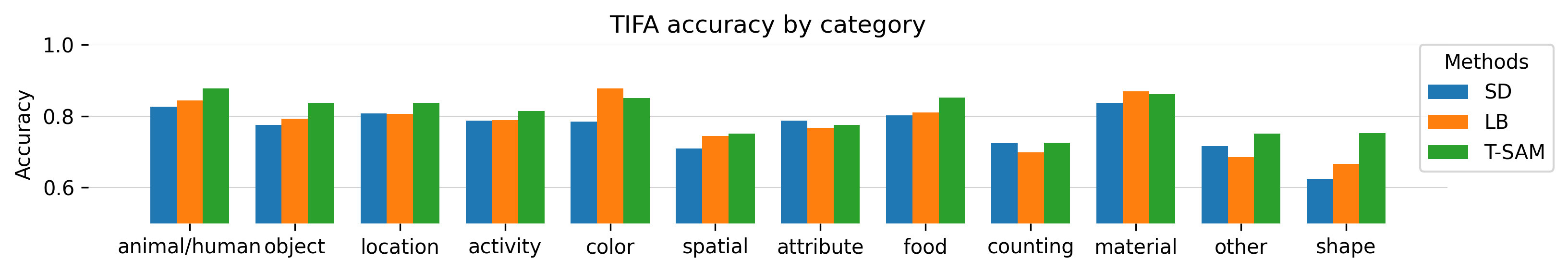}
    \caption{Accuracy for each question type in the TIFA benchmark. Our method achieves the best performance in most categories. The \textit{attribute} includes properties such as large, small, young, etc. }
    \label{fig:tifa_comp}
    \vspace{-1mm}
\end{figure*}
\begin{figure*}[ht]
    \centering
    \includegraphics[width=0.83\linewidth]{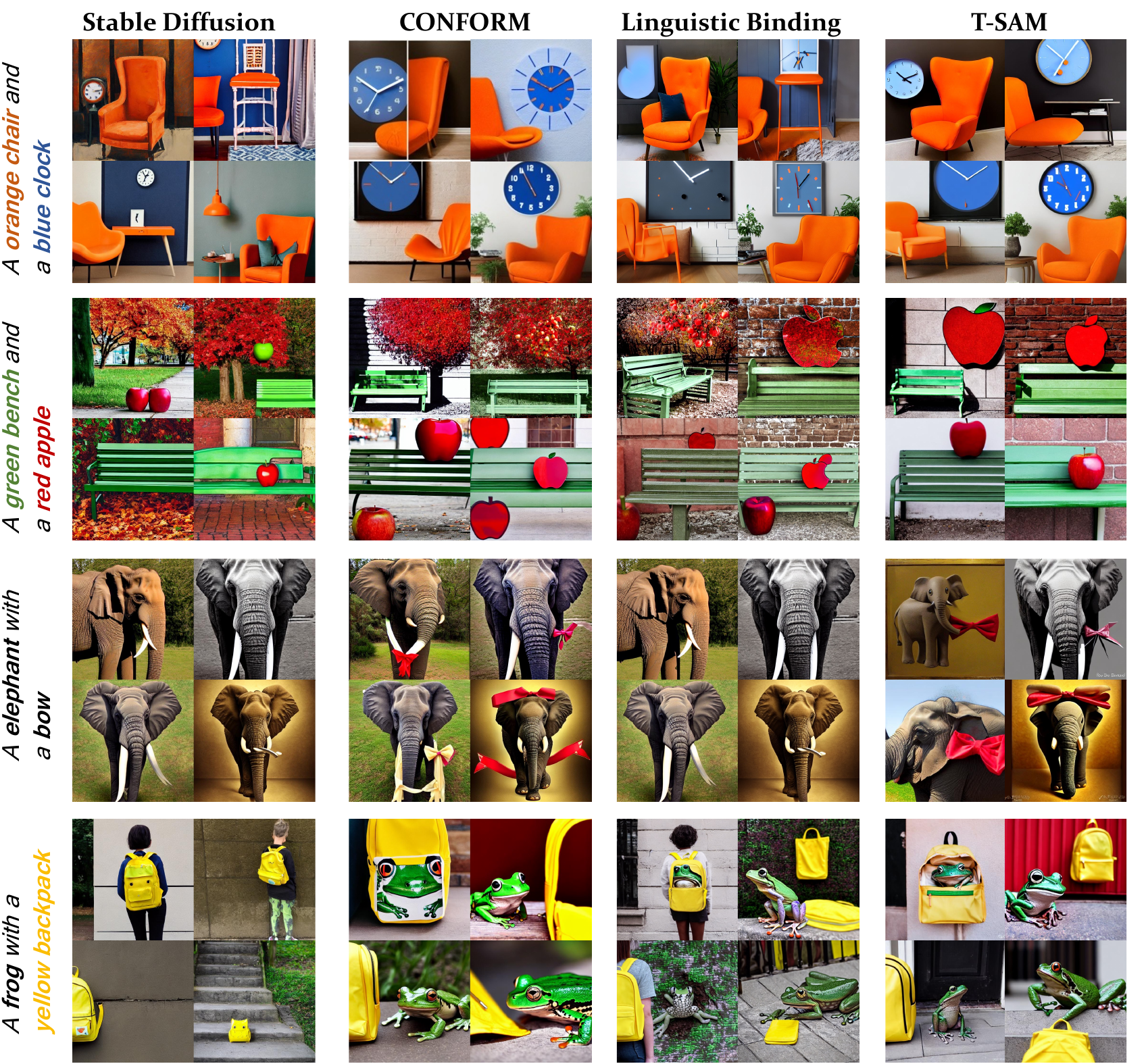}
    \caption{Comparison of our method (T-SAM) with recent state-of-the-art methods on prompts from \textit{Objects} and \textit{Animals-Objects}. The images corresponding to the same position across different methods are generated using the same seed. Best in zoom.}
    \label{fig:fixed_template_images}
    \vspace{-2mm}
\end{figure*}

\noindent\textbf{Evaluation metrics.}
% \paragraph{Evaluation metrics.}
To quantitatively evaluate the accuracy of generated images, we use two metrics: TIFA scores~\cite{hu2023tifa} (for TIFA benchmark) and CLIP similarity scores.
TIFA scores measure how well generated images reflect the text prompts. In TIFA, questions for each prompt are generated by GPT-3.5~\cite{gpt}, and a vision-language model~\cite{li2022mplug} provides answers.
For CLIP~\cite{radford2021learning} scores, we follow the protocol in \cite{chefer2023attend} to evaluate both image-text and text-text similarity. Image-text similarity includes two measures: \textit{full-prompt similarity}, which evaluates overall alignment with the prompt, and \textit{minimum object similarity}, defined as the lowest similarity score between the generated image and the two main subjects in the prompt. Additionally, we use BLIP~\cite{li2022blip} to generate captions for the images, comparing the input prompt with these captions using CLIP to assess text-text similarity.

\noindent\textbf{Baselines.}
We compare our method with SD, Linguistic-Binding (LB)~\cite{rassin2024linguistic}, Attend-n-Excite (A\&E)~\cite{chefer2023attend} and {\textit{\footnotesize CONFORM}} \cite{meral2024conform}.
LB depending on external parsers (SpaCy \cite{honnibal2017spacy}) is limited to attribute-binding tasks. This method can be applied to both the TIFA benchmark and the Attend-n-Excite (\textit{Objects} and \textit{Animals-Objects}) prompt sets, though they are used only when the text parser identifies (modifier, entity-noun) pairs within the given prompts.
A\&E and {\textit{\footnotesize CONFORM}} require manual token selection per prompt, restricting their use to the fixed-template prompts (\textit{Objects} and \textit{Animals-Objects}) due to the high cost of selecting token indices for diverse prompts. We reproduce the images with SD, {\textit{\footnotesize CONFORM}}, and LB based on SD v1.5 and using the same seeds.

\subsection{Results}
\begin{table*}[t]
    \centering
    \vspace{-2mm}
    \begin{minipage}{\textwidth}  % Ensures table fits within text width
        \centering
        \caption{CLIP similarity scores. Average CLIP similarities between the text prompts and the images generated with 64 different seeds.}
        \begin{tabular}{lccccc|cc}
            \toprule
             & \multirow{3}{*}{} & \multicolumn{4}{c}{\textbf{Image-Text Prompt}} &\multicolumn{2}{c}{\multirow{2}{*}{\textbf{Prompt-Caption}}}  \\
            &&\multicolumn{2}{c}{\textbf{Full Prompt}} & \multicolumn{2}{c}{\textbf{Minimum Object}} & & \\ 
            \cmidrule(lr){3-4} \cmidrule(lr){5-6} \cmidrule(lr){7-8}
            &External Info.&  Objects & Animals-Objects &  Objects & Animals-Objects & Objects & Animals-Objects\\ 
            \midrule
            SD & \xmark &0.34 & 0.34 & 0.25 & 0.26 & 0.76 & 0.80 \\
            LB & \checkmark&\textbf{0.36} & 0.35 & 0.27 & 0.27  & \underline{0.80} & 0.83\\
            % StructureD. &\checkmark& & & & & 0.76 & 0.78\\
            CONFORM & \checkmark&\textbf{0.36} & \underline{0.36}  & \textbf{0.28} & \textbf{0.28} & \textbf{0.81} & \textbf{0.85} \\
            A\&E & \checkmark&\textbf{0.36} & 0.35 & 0.27 & 0.26 & 0.81 & 0.83\\
            \midrule
            \textbf{T-SAM}&\xmark & \textbf{0.36} & \textbf{0.37} &  \textbf{0.28} &  \textbf{0.28} & \underline{0.80} & \textbf{0.85} \\
            \bottomrule
        \end{tabular}
        \label{tab:merged_clip_scores}
    \end{minipage}
    \vspace{-2mm}
\end{table*}

\begin{figure}
    \centering
    \includegraphics[width=1.0\linewidth]{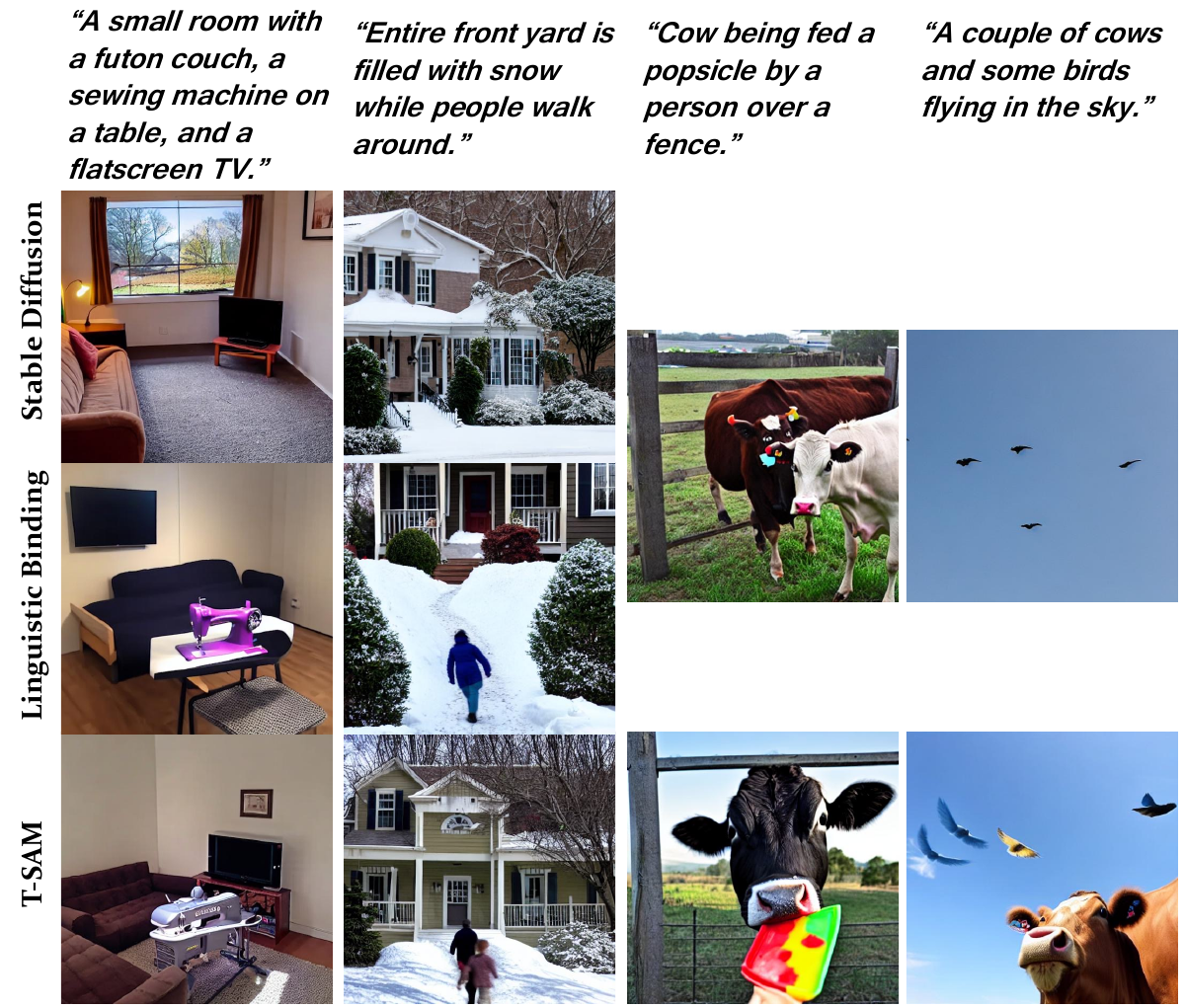}
    \caption{Qualitative comparison using prompts from MSCOCO contained in TIFA benchmark. Best in zoom.}
    \label{fig:tifa_figures}
    \vspace{-2mm}
\end{figure}
\noindent\textbf{Quantitative results.}
Table~\ref{tab:tifa_scores} shows the evaluation results in TIFA benchmark and Figure~\ref{fig:tifa_comp} illustrates the breakdown of TIFA scores across question types. We highlight that our approach (T-SAM) outperforms SD and LB on complex syntactic prompts, where other baselines are inapplicable.  While LB demonstrates improvements over SD in the \textit{color}, \textit{shape}, and \textit{material} categories, which are closely related to attribute-binding tasks, it fails to enhance accuracy in the \textit{activity} and \textit{counting} categories. This underscores LB’s inability to capture diverse word relationships, even when using external text parsers.
In contrast, our method shows improvements over SD in nearly all categories including \textit{color}, \textit{shape},  \textit{counting}, and \textit{activity}, except for \textit{attribute} (e.g., properties such as large, small, young), which differ from the term ``attribute'' used in this study, demonstrating its versatility across a wide range of word relationships. Additionally, it achieves higher CLIP similarity scores than the baselines, confirming its superior semantic alignment.

In the structured templates (\textit{Objects} and \textit{Animals-Objects}), T-SAM performs comparably to the state-of-the-art {\textit{\footnotesize CONFORM}} , which, unlike our approach, requires manually defined token indices for positive and negative groups. Our method outperforms LB and A\&E, which also rely on external inputs. This demonstrates that extracting syntactic information from text self-attention maps can be more effective than relying on text parsers or manually selecting tokens for optimization.
An additional ablation study is provided in Appendix B.

\noindent\textbf{Qualitative results.}
Figure~\ref{fig:tifa_figures} presents images generated by T-SAM and the baselines, showcasing effectiveness of T-SAM with complex prompts, such as MSCOCO captions from the TIFA benchmark. Our approach successfully generates multiple elements, including \textit{a sewing machine} in the first prompt, \textit{people} in the second, \textit{popsicle} in the third, and \textit{cow} in the fourth. Notably, SD misses some elements. And LB generates identical images to SD when no relations are extracted from the text parser, as seen in the third and fourth prompts, highlighting its limited generalizability.

In the structured templates (\textit{Objects} and \textit{Animals-Objects}; see Fig.~\ref{fig:fixed_template_images}),  T-SAM either outperforms or performs comparably to the baselines. SD often omits objects (\eg, \textit{clock}, \textit{apple}, or \textit{bow}) or misbinds attributes (e.g., \textit{blue clock}), while {\textit{\footnotesize CONFORM}} , LB, and our method more reliably generate specified elements in the prompts. However, {\textit{\footnotesize CONFORM}}  and LB have limitations. LB sometimes has fidelity issues, such as missing a \textit{clock} in the first image of the first prompt. Its effectiveness is also limited by its focus on attribute binding; for prompts without modifiers (e.g., \textit{An elephant with a bow}), LB generates images exactly same as SD. Conversely, {\textit{\footnotesize CONFORM}}  sometimes introduces overly strict separations in the image, as seen in the first example for the first prompt. In contrast, our method is broadly applicable and achieves notable improvements across diverse cases without these artificial separations. This advantage likely comes from our method’s use of smoother linguistic structures from text attention maps, rather than binary categorization of positive and negative pairs used in {\textit{\footnotesize CONFORM}}.

\section{Conclusion}
To enhance fidelity in text-to-image diffusion models, we improved cross-attention maps by aligning their similarity matrix with text self-attention maps. Our approach is based on two insights: (1) similar text embeddings produce similar cross-attention maps, but (2) syntactic relations are missed in embeddings but captured by text self-attention maps. Our method enabled cross-attention to better capture syntactic structure, significantly improving text-to-image fidelity across a range of sentence structures, without requiring external resources.

\newpage
{
    \small
    \bibliographystyle{ieeenat_fullname}
    \bibliography{main}
}

\clearpage
\setcounter{page}{1}
\appendix
\setcounter{table}{0}
\setcounter{equation}{0}
\setcounter{figure}{0}
\renewcommand{\thetable}{\Alph{table}}
\renewcommand{\thefigure}{\Alph{figure}}
\maketitlesupplementary
\newtheorem{definition}{Definition}
\newtheorem{corollary}{Corollary}
\newtheorem{theorem}{Theorem}

\newcommand{\mean}[1]{\mathbb{E}[#1]}
\section{Proofs}
\subsection{Notation}
\begin{table}[ht]
\centering
\begin{tabular}{l p{15cm}} % Adjust the width as needed
\toprule
\textbf{Symbol} & \textbf{Definition and Properties} \\ 
\midrule
${\Dc}$ & embedding dimension per head in cross-attention layers\\
$\Hc$ & number of heads in cross-attention layers\\
$\Nc$ &query sequence length in cross-attention layers\\
$\De$ &embedding dimension per head in text encoder\\
$\He$ &number of heads in text encoder\\
$s$ & text sequence length\\
$\mathbf{q}^{(\ell)}_a$ & $\in\rr{\Hc\Dc}, a = 1,\cdots,\Nc$, query vectors at layer $\ell$ cross-attention layer  \\
$\mathbf{k}^{(\ell)}_i$ & $\in\rr{\He\De}$, $i = 1,\cdots, s$, text embeddings  \\
$W^{(\ell,h)}_{\text{c}}$ & $\in\rr{\Hc\Dc\times \Hc\Dc}$, projection parameter matrix in cross-attention layer $\ell$ and head $h$. It is related to the product key and query projection parameters ($\in\mathbb{R}^{\Hc\times\Dc\times\Hc\Dc}$) via $W^{(\ell,h)}_{\text{c}}=W^{(\ell,h)\top}_{\text{q}} W^{(\ell,h)}_{\text{k}}$\\
$W^{(\ell,h)}_{\text{v}}$ &$\in\rr{\De\times \He\De}$, value projection matrix in text encoder self-attention layer $\ell$ and head $h$\\
$W^{(\ell)}_{\text{out}}$ &$\in\rr{\He\De\times \He\De}$, out projection matrix in text encoder self-attention layer $\ell$\\
$A^{(\ell,h)}$ &  $\in \rr{\Nc\times s}$, cross-attention maps at layer $\ell$ and head $h$. The elements are denoted by $A^{(\ell,h)}_{ai}$ \\
$\mathbf{e}^{(\ell)}_i$ & $\in\rr{\He\De}$, $i = 1,\cdots, s$, text dense vectors  in the text-encoder layer $\ell$ \\
$T^{(\ell,h)}$& $\in \mathbb{R}^{s\times s}$, text self-attention matrix at layer $\ell$ and head $h$ of the text encoder\\
$\epsilon$ &$\ll 1$, the inverse ratio of \texttt{<bos>} attention weight to the sum of attention weights of the rest of the sequence\\
\bottomrule
\end{tabular}
\caption{Table of Notations}
\end{table}

\subsection{The Big O Notation}
Based on the empirical observations, we consider the situation where \emph{attention sink} occurs both in the text encoder and in the cross-attention layers of the diffusion model:  the attention weights for the \bos token are much higher than the rest of the sequence:
\begin{align}
\text{self-attention in the text encoder :}\quad\frac{\sum^{i}_{j\neq 1}T^{(\ell,h)}_{ij}}{T^{(\ell,h)}_{i1}}&<\epsilon \ll 1, \qquad i=2,\cdots,s,\\
\text{cross-attention in the diffusion model:}\quad\frac{\sum^{s}_{j\neq 1}A^{(\ell,h)}_{aj}}{A^{(\ell,h)}_{a1}}&<\epsilon \ll 1, \qquad a=1,\cdots,\Nc.
\end{align}
In practice $\epsilon\sim 0.1$ or smaller in the middle layers of U-Net and the later layers of CLIP text encoder. Our approach for calculating the approximate quantities in the limit of small $\epsilon$ is \emph{perturbation theory}: we assume that the variables of the problem, such as $T^{(\ell,h)}$ can be written as a power series in a small parameter $\epsilon$:
\begin{equation}
    T^{(\ell,h)}= \sum^\infty_{n=0}  T^{(\ell,h)(0)}+  \epsilon T^{(\ell,h)(1)}+ \epsilon^2 T^{(\ell,h)(2)}+\cdots.
\end{equation}
In this context, $\ord{\epsilon}$ mean terms that are linear or higher order in $\epsilon$. If $\epsilon$ is sufficiently small, the first few term give a good approximation to the true variable.

\subsection{Proof of Proposition 1}
In the following, we suppress the $(\ell,h)$ superscript in quantities $A^{(\ell,h)},\Sigma^{(\ell,h)},\mu^{(\ell,h)},W_{\text{c}}^{(\ell,h)}, \mathbf{k}_{i}^{(\ell)},\mathbf{q}_{i}^{(\ell)}$ defined below to reduce clutter.
\setcounter{propx}{0}
Consider the similarity matrix of the form:
    \begin{equation}
       \cos({A}_{ai},
       A_{aj}):=\frac{\sum_{a=1}^{\Nc} {A}_{ai} A_{aj}}{\big(\sum_{a=1}^{\Nc}A^{(\ell,h)2}_{ai}\big)^{\frac12}\big(\sum_{a=1}^{\Nc}A^{(\ell,h)2}_{aj}\big)^{\frac12}}, 
    \end{equation}
where
\begin{align}
    A_{ai} := \frac{\exp(\Omega_{ai})}{\sum_{j=1}^s\exp(\Omega_{aj})},\;
    \Omega_{ai}:=\mathbf{q}_a^{\top} W^{}_\text{c} \mathbf{k}^{}_i.
\end{align}

\begin{prop}
    If $A\in\mathbb{R}^{ \Nc\times s}$ is a cross-attention map defined in eq. \eqref{eq:cross_attn}, then under the assumptions i, ii, 
 and iii described below, the similarity matrix can be written in terms of key vectors $\mathbf{k}_i\in \mathbb{R}^{\Hc\Dc}$ as
\begin{align}
    \cos(A_i,A_j)= 
   \;\; \exp\Big(-\frac{1}{2}(\mathbf{k}_i-\mathbf{k}_j)^\top  W^2(\mathbf{k}_i-\mathbf{k}_j)\Big),
\end{align}
up to terms of at least $\mathcal{O}(1/\sqrt{\Nc})$ and $\mathcal{O}(\epsilon)$, where $W^2 :=W^\top_\text{c} \Sigma W_\text{c}$ and $\Sigma^{(\ell)}\in\mathbb{R}^{\Hc\Dc \times \Hc\Dc}$ is the covariance matrix of query vectors and $W_\text{c}\in\mathbb{R}^{\Hc\Dc \times \Hc\Dc}$ is a parameter.
\end{prop}

\begin{proof}
If queries are iid samples of some distribution $\mathbf{q}_a^{(\ell)}\sim p_{\mathbf{q}}$ with finite mean and variance, we can use the Central Limit Theorem to write the cosine similarity as 
   \begin{align}\label{cos2}
       \cos(A_i,A_j)&:=\frac{\mathbb{E} [{A}_{ai} A_{aj}]+\ord{\frac{1}{\sqrt{N_{\text{c}}}}}}{\big(\mathbb{E}[A^{(\ell,h)2}_{ai}]+\ord{\frac{1}{\sqrt{N_{\text{c}}}}}\big)^{\frac12}\big(\mathbb{E}[A^{(\ell,h)2}_{aj}]+\ord{\frac{1}{\sqrt{N_{\text{c}}}}}\big)^{\frac12}},\\
       &=\frac{\mathbb{E} [{A}_{ai} A_{aj}]}{\big(\mathbb{E}[A^{2}_{ai}]\big)^{\frac12}\big(\mathbb{E}[A^{2}_{aj}]\big)^{\frac12}}+\ord{\frac{1}{\sqrt{N_{\text{c}}}}}.
    \end{align}

\begin{assumptionx}
    The query sequence length $\Nc$ is large enough so that deviations from true mean can be approximated by the first term in $1/\Nc$ expansion, and the corrections from the dependence between samples appear at higher orders in the expansion.
\end{assumptionx}
As a first approximation to the distribution of queries, consider a statistical model where the query vectors are jointly normally distributed:

\begin{equation}\label{q dist}
    \mathbf{q}_a\sim\mathcal{N}(\mu,\Sigma),\qquad \mu\in \rr{\Hc\Dc},\, \Sigma\in \rr{\Hc\Dc\times \Hc\Dc}.
\end{equation}
This is strictly true at the first denoising step. Moreover, if the true distribution remains close to Gaussian, the corrections from the distribution can  in principle be perturbatively calculated and added accordingly. Therefore, the assumption above may not be interpreted as a restriction, but as a first (and good) approximation to the true distribution.
\begin{assumptionx}
    Query vectors $\mathbf{q}_{a}\in\mathbb{R}^{\Hc\Dc}$ are jointly Gaussian as in \eqref{q dist}.
\end{assumptionx}

Note that the \emph{attention scores} $\Omega_{ai} = \mathbf{q}_a^{\top} W_\text{c} \mathbf{k}_i$ are now gaussian variables with
\begin{equation}
    \mean{\Omega_{ai}} = \mu^\top W_{\text{c}}\mathbf{k}_i:=\mu_i,\qquad \text{Var}[\Omega_{ai}] = \mathbf{k}^\top_iW^\top_{\text{c}}W_{\text{c}}\mathbf{k}_i:=\sigma^2_{i}.
\end{equation}
\begin{assumptionx}\label{assump stat}
    We empirically observe that i) $\mu_1\gg \mu_i,\, ii) \mu_1\gg \sigma_1, \,iii) \sigma_1\approx \sigma_i$ for $,i=2,\cdots,s$, such that
    \begin{equation}
        e^{\mu_i-\mu_1}\sim\ord{\epsilon}.
    \end{equation}
\end{assumptionx}
Writing cross-attention probabilities in terms of attention scores, we have
\begin{equation}
    A_{ai} =   \frac{e^{\Omega_{ai}}}{e^{\Omega_{a1}}+\sum_{m=2}^se^{\Omega_{am}}}= \frac{e^{\Omega_{ai}-\Omega_{a1}}}{1+\sum_{m=2}^se^{\Omega_{am}-\Omega_{a1}}}.
\end{equation}

Note that since attention scores are Gaussian,
\begin{equation}
    \mathbb{P}[e^{\Omega_{ai}-\Omega_{a1}}<\epsilon/s]=\Phi\big(\frac{\log(\epsilon/s)-\mu_i+\mu_1}{\sqrt{\sigma_1^2+\sigma_i^2}}\big).
\end{equation}
Therefore, if assumption  \ref{assump stat} holds, for some large enough $\mu_1$, we can have $\mathbb{P}[e^{\Omega_{ai}-\Omega_{ai}}<\epsilon/s]>1-\epsilon^3$. This means that the sum of attention probabilities of all non-\texttt{<bos>} tokens does not exceed $\epsilon$ with the probability of at least $1-\epsilon^3$. As a result, we have
\begin{align}
    A_{ai} &= e^{\Omega_{ai}-\Omega_{a1}}+\ord{\epsilon^2},\\
    A_{ai}A_{aj}& = e^{\Omega_{ai}+\Omega_{aj}-2\Omega_{a1}}+\ord{\epsilon^3},\label{approx}
\end{align}
with high probability. To evaluate the cosine similarity, we need to compute expectations:

\begin{align}
    \frac{\mathbb{E} [{A}_{ai} A_{aj}]}{\big(\mathbb{E}[A^{2}_{ai}]\big)^{\frac12}\big(\mathbb{E}[A^{2}_{aj}]\big)^{\frac12}}&=
    \frac{\mathbb{E} [e^{\Omega_{ai}+\Omega_{aj}}]+\ord{\epsilon^3}}{\big(\mathbb{E}[e^{2\Omega_{ai}}]+\ord{\epsilon^3}\big)^{\frac12}\big(\mathbb{E}[e^{2\Omega_{aj}}]+\ord{\epsilon^3}\big)^{\frac12}}.
\end{align}
We can evaluate this expression using the well-known formula of the moment-generating function of Gaussian distribution:
\begin{lemma}\label{gauss}
    If $\mathbf{q}_a\sim\mathcal{N}(\mu,\Sigma)$ and $\mathbf{r}\in\mathbb{R}^{\Hc\Dc}$, then
    \begin{equation}
        \mean{e^{\mathbf{q}\cdot \mathbf{r}}} = \exp(\mathbf{r}\cdot\mu+\frac{1}{2}\mathbf{r}\cdot \Sigma \cdot\mathbf{r}).
    \end{equation}
\end{lemma}
Define $\mathbf{r}_{ij} = W_\text{c} (\mathbf{k}_i+\mathbf{k}_j-2\mathbf{k}_1)$:
\begin{equation}\label{exp by r}
   \frac{\mathbb{E} [{A}_{ai} A_{aj}]}{\big(\mathbb{E}[A^{2}_{ai}]\big)^{\frac12}\big(\mathbb{E}[A^{2}_{aj}]\big)^{\frac12}}   = \frac{\exp(\mu^\top \mathbf{r}_{ij}+\frac{1}{2}\mathbf{r}_{ij}^\top\Sigma \mathbf{r}_{ij})}{\exp(\frac12\mu^\top \mathbf{r}_{ii}+\frac{1}{4}\mathbf{r}_{ii}\Sigma\mathbf{r}_{ii})\exp(\frac12\mu^\top \mathbf{r}_{jj}+\frac{1}{4}\mathbf{r}_{jj}\Sigma\mathbf{r}_{jj})}+\ord{\epsilon}.
\end{equation}
Here, we used the fact that each of exponentials are $\sim\ord{\epsilon^2}$ to simplify the correction terms to $\ord{\epsilon}$. When applying Lemma \ref{gauss}, one might worry that the integration includes regions of $\mathbb{R}^{\Hc\Dc}$ that the approximation \eqref{approx} fails. Although this is a valid point, the total probability of such regions is $\epsilon^3$ by assumption,  which is at the order of correction terms.

Simplifying \eqref{exp by r} gives
\begin{equation}
   \frac{\mathbb{E} [{A}_{ai} A_{aj}]}{\big(\mathbb{E}[A^{2}_{ai}]\big)^{\frac12}\big(\mathbb{E}[A^{2}_{aj}]\big)^{\frac12}}   = \exp(\frac{1}{2}\mathbf{r}_{ij}^\top\Sigma \mathbf{r}_{ij}-\frac{1}{4}\mathbf{r}_{ii}^\top\Sigma \mathbf{r}_{ii}-\frac{1}{4}\mathbf{r}_{jj}^\top\Sigma \mathbf{r}_{jj})+\ord{\epsilon}.
\end{equation}
Substituting the definition of $\mathbf{r}_{ij},$ we get
\begin{equation}
   \frac{\mathbb{E} [{A}_{ai} A_{aj}]}{\big(\mathbb{E}[A^{2}_{ai}]\big)^{\frac12}\big(\mathbb{E}[A^{2}_{aj}]\big)^{\frac12}}   = \exp(-\frac{1}{2}(\mathbf{k}_i-\mathbf{k}_j)^\top W^\top_\text{c} \Sigma W_\text{c} (\mathbf{k}_i-\mathbf{k}_j)+\ord{\epsilon}.
\end{equation}

\end{proof}

\subsection{Proof of Proposition 2}
\begin{prop} Consider a self-attention layer with output $\mathbf{o}^{(\ell,h)}_i\in\mathbb{R}^{\De}$ defined as
\begin{align} 
\mathbf{o}^{(\ell,h)}_i &= \sum_{j=1}^i T^{(\ell,h)}_{ij}W_{\text{v}}^{(\ell,h)} \mathbf{e}^{(\ell)}_j,
\end{align} 
where $W_{\text{v}}^{(\ell,h)}\in\mathbb{R}^{\De\times\He\De}$ is a parameter, $ \mathbf{e}_i \in \mathbb{R}^{\He\De}$, and $T_{ij}^{(\ell,h)}$ is the self-attention matrix. 
    Define $R\in\mathbb{R}^{s\times s}$ as
    \begin{equation}
        R_{ij}=  \mathbf{e}^{(\ell)\top}_iW_{\text{v}}^{(\ell,h)\top} W_{\text{v}}^{(\ell,h)} \mathbf{e}^{(\ell)}_j,
    \end{equation}
and suppose it has the property
\begin{equation}
    \frac{|R_{mn}|}{R_{11}}\sim\mathcal{O}(1/\epsilon),\qquad \frac{|R_{1m}|}{R_{11}}\sim\mathcal{O}(1),\quad 1<m,n\leq s,
\end{equation}
where $\mathbf{e}^{(\ell)}_1$ is the bos embedding. Then the following holds:
\begin{equation}
    \cos(\mathbf{o}^{(\ell,h)}_i,\mathbf{o}^{(\ell,h)}_j)=1-\mathcal{O}(\epsilon).
\end{equation}
\end{prop}

\begin{proof}
In terms of the matrix $R$, we have

\begin{equation}
 \cos(\mathbf{o}^{(\ell,h)}_i,\mathbf{o}^{(\ell,h)}_j)=\frac{\mathbf{o}^{(\ell,h)}_i\cdot\mathbf{o}^{(\ell,h)}_j)}{\|\mathbf{o}^{(\ell,h)}_i\|\|\mathbf{o}^{(\ell,h)}_j\|} =\frac{\sum_{m=1}^i\sum_{n=1}^j  T^{(\ell,h)}_{im}T^{(\ell,h)}_{jn}R_{mn}}{\big(\sum_{m=1}^i\sum_{n=1}^j  T^{(\ell,h)}_{im}T^{(\ell,h)}_{jn}R_{mn}\big)^\frac12\big(\sum_{m=1}^i\sum_{n=1}^j  T^{(\ell,h)}_{im}T^{(\ell,h)}_{jn}R_{mn}\big)^{\frac12}}.
\end{equation}
By separating the contributions from the \texttt{<bos>} token, we can write the numerator as
\begin{equation}
    \sum_{m=1}^i\sum_{n=1}^j  T^{(\ell,h)}_{im}T^{(\ell,h)}_{jn}R_{mn}=T^{(\ell,h)}_{i1}T^{(\ell,h)}_{j1}R_{11}+\sum_{m'=2}^iT^{(\ell,h)}_{im'}T^{(\ell,h)}_{j1}R_{m'1}+\sum_{n'=2}^jT^{(\ell,h)}_{i1}T^{(\ell,h)}_{jn'}R_{1n'}+\sum_{m'=2}^i\sum_{n'=2}^jT^{(\ell,h)}_{im'}T^{(\ell,h)}_{jn}R_{m'n'},
\end{equation}
where $m',n' =2,\cdots,s$. Using $T^{(\ell,h)}_{im'}\sim\ord{\epsilon}$, $T^{(\ell,h)}_{i1}= 1-\ord{\epsilon}$and the conditions on the $R$ matrix stated in Proposition 2, we can show that the first term (which is necessarily positive) has a larger norm than the rest:
\begin{equation}
   \sum_{m=1}^i\sum_{n=1}^j T^{(\ell,h)}_{im}T^{(\ell,h)}_{jn}R_{mn}= R_{11} +\ord{\epsilon}.
\end{equation}
We can use this property to perform a Taylor expansion in the denominator and keep up to the linear term in $\epsilon$.
At this order of approximation, we can use $\frac{1}{1+x}\approx 1-x$ since we are keeping only the first terms in Taylor expansion. After some algebra, the expression drastically simplifies to
\begin{equation}
\cos(\mathbf{o}^{(\ell,h)}_i,\mathbf{o}^{(\ell,h)}_j)\approx 1+\sum_{m'=2}^i\sum_{n'=2}^j\Big(\frac{T^{(\ell,h)}_{im'}T^{(\ell,h)}_{jn'}}{T^{(\ell,h)}_{i1}T^{(\ell,h)}_{j1}}-\frac12
\frac{T^{(\ell,h)}_{im'}T^{(\ell,h)}_{in'}}{T^{(\ell,h)2}_{i1}}-\frac12\frac{T^{(\ell,h)}_{jm'}T^{(\ell,h)}_{jn'}}{T^{(\ell,h)2}_{j1}}\Big)
R_{m'n'}.
\end{equation}
First, note that the expression in brackets is $\sim\ord{\epsilon^2}$ and since $R_{m'n'}\sim\ord{1/\epsilon}$ (both claims from empirical evidence), the total expression will be $1+\ord{\epsilon}$. Furthermore, matrix $R$ is positive semi-definite by definition. Performing an SVD decomposition $R=U\kappa U^\top$ and absorbing the terms in brackets in $U$ matrices, one can use the Cauchy-Schwarz inequality to show that the correction term above is negative. In conclusion, we have
\begin{equation}
    \cos(\mathbf{o}^{(\ell,h)}_i,\mathbf{o}^{(\ell,h)}_j)= 1- c\epsilon+\ord{\epsilon^2},
\end{equation}
for some positive $c\sim\ord{1}$.
\end{proof}

\subsection{An extension to Proposition 2}
In this section, we extend the result of Proposition 2 by considering the full self-attention layer of the text encoder. Proposition 2 was only concerned with products of keys, queries, and values. However, a self-attention layer typically includes an output-projection linear layer and a skip connection.  Here, we explore the effect of these two components and compare the cosine similarities of input text embeddings versus those of outputs. We show that based on practical assumptions that are valid in the later layers of CLIP text encoder, the output cosine similarities are close to input cosine similarities.
Consider the output of an attention head at layer $\ell$:
\begin{align}
    \mathbf{e}_{i}^{(\ell)\text{out}} &= \mathbf{e}_{i}^{(\ell)} +W^{(\ell)}_{\text{out}}\text{concat}\big[\sum_{m=1}^i T^{(\ell,1)}_{im}W^{(\ell,1)}_\text{v}{\mathbf{e}}_{m}^{(\ell)},\cdots, \sum_{m=1}^iT^{(\ell,\He)}_{im} W^{(\ell,\He)}_\text{v}{\mathbf{e}}_{m}^{(\ell)}\big]
\end{align}
in which $\mathbf{e}_i^{(\ell)\text{out}},\mathbf{e}_i^{(\ell)}\in\mathbb{R}^{\He\De}$ for all $i=1,\cdots,s$. Here, $W^{(\ell)}_{\text{out}}\in\mathbb{R}^{\He\De\times\He\De}$ is the out-projection layer.

Define the average attention probabilities:
\begin{equation}
    \tau_i^{(\ell,h)} = \frac{1}{i}\sum_{m=2}^iT_{im}^{(\ell,h)},\qquad i =2,\cdots, s.
\end{equation}
By adding and subtracting terms proportional to $\tau_i^{(\ell,h)}$ we have
\begin{align}
    \mathbf{e}_{i}^{(\ell)\text{out}} &= \mathbf{e}_{i}^{(\ell)} +\underbrace{W^{(\ell)}_{\text{out}}\text{concat}\Big[T_{i1}^{(\ell,h)}W^{(\ell,1)}_\text{v}\mathbf{e}_1^{(\ell)}+\tau_i^{(\ell,1)}\sum_{m=2}^i W^{(\ell,1)}_\text{v}{\mathbf{e}}_m^{(\ell)},\cdots, T_{i1}^{(\ell,\He)}W^{(\ell,\He)}_\text{v}\mathbf{e}_1^{\ell}+\tau_i^{(\ell,\He)}\sum_{m=2}^iW^{(\ell,\He)}_\text{v}{\mathbf{e}}_m^{(\ell)}\Big]}_{:=\mathbf{e}_i^{\prime(\ell)}}\nonumber\\
    \qquad\qquad\qquad&+\underbrace{W^{(\ell)}_{\text{out}}\text{concat}\Big[\sum_{m=2}^i(T^{(\ell,h)}_{im}-\tau^{(\ell,1)}_i) W^{(\ell,1)}_\text{v}{\mathbf{e}}_{m}^{(\ell)}
    ,\cdots, 
    \sum_{m=2}^i(T^{(\ell,\He)}_{im}-\tau^{(\ell,\He)}_i) W^{(\ell,\He)}_\text{v}{\mathbf{e}}_{m}^{(\ell)}
    \Big]}_{:=\delta\mathbf{e}_i^{(\ell)}}.
\end{align}
We empirically observe that
\begin{equation}
    \|\mathbf{e}_i^{(\ell)}\|\sim\mathcal{O}(1/\epsilon),\qquad\|\mathbf{e}_i^{\prime(\ell)}\|\sim\mathcal{O}(1),\qquad \|\delta \mathbf{e}_i^{(\ell)}\|\sim\mathcal{O}(\epsilon).
\end{equation}
The result of the above conditions is
\begin{align}
     \big|\mathbf{e}_{i}^{(\ell)\text{out}}\cdot  \mathbf{e}_{j}^{(\ell)\text{out}}-(\mathbf{e}_i^{(\ell)}+ \mathbf{e}_i^{\prime(\ell)})\cdot (\mathbf{e}_j^{(\ell)}+\mathbf{e}_j^{\prime(\ell)})\big|&=(\mathbf{e}_i^{(\ell)}+\mathbf{e}_i^{\prime(\ell)})\cdot\delta \mathbf{e}_j^{(\ell)}+
     (\mathbf{e}_j^{(\ell)}+\mathbf{e}_j^{\prime(\ell)})\cdot\delta \mathbf{e}_i^{(\ell)}+
     \delta\mathbf{e}_i^{(\ell)}\cdot\delta \mathbf{e}_j^{(\ell)}\nonumber\\
     &\sim\ord{1}.
\end{align}
Thus, the cosine similarity will be
\begin{align}
    \cos(\mathbf{e}_i^{(\ell)\text{out}},\mathbf{e}_j^{(\ell)\text{out}}) &= \frac{(\mathbf{e}_i^{(\ell)}+\mathbf{e}_i^{\prime(\ell)})\cdot (\mathbf{e}_j^{(\ell)}+\mathbf{e}_j^{\prime(\ell)})+\ord{1}}{\Big(\|\mathbf{e}_i^{(\ell)}+\mathbf{e}_i^{\prime(\ell)}\|^2+\ord{1}\Big)^{\frac12}\Big(\|\mathbf{e}_j^{(\ell)}+\mathbf{e}_j^{\prime(\ell)}\|^2+\ord{1}\Big)^{\frac12}}\\
    &=
    \frac{(\mathbf{e}_i^{(\ell)}+\mathbf{e}_i^{\prime(\ell)})\cdot (\mathbf{e}_j^{(\ell)}+\mathbf{e}_j^{\prime(\ell)})}{\Big(\|\mathbf{e}_i^{(\ell)}+\mathbf{e}_i^{\prime(\ell)}\|^2\Big)^{\frac12}\Big(\|\mathbf{e}_j^{(\ell)}+\mathbf{e}_j^{\prime(\ell)}\|^2\Big)^{\frac12}}+\ord{\epsilon^2}\\
    &=\cos(\mathbf{e}_i^{(\ell)}+\mathbf{e}_i^{\prime(\ell)},\mathbf{e}_j^{(\ell)}+\mathbf{e}_j^{\prime(\ell)})+\ord{\epsilon^2}.
\end{align}
This result shows that the similarity between output text embeddings corrected by terms that depend only on the averages of attention probabilities ($\tau_i$).

\section{Details of the Empirical Study}
\paragraph{The experiment setup for Figure \ref{fig:tifa} in Section \ref{sec:evidence}.}

We use 144 prompts from the \textit{Animals-Objects} in \cite{chefer2023attend}, featuring two distinct objects, to address the missing object scenario illustrated in Figure \ref{fig:missing_objects_tifa}. In addition, we use 107 prompts from the same dataset, incorporating both attributes and objects, to explore the attribute binding case shown in Figure \ref{fig:attr_binding_tifa}.
To assess image correctness, we use TIFA~\cite{hu2023tifa}, which evaluates how well the generated images reflect the text prompts. In TIFA, questions (\eg, ``\textit{is there a green backpack?}'') for each text prompt are generated by GPT-3.5~\cite{gpt} and a vision-language model~\cite{li2022mplug} is used to provide binary or multiple-choice answers. A case is considered incorrect if the TIFA score for any question regarding object presence or attribute binding is incorrect. 

\paragraph{Text embeddings vs. text self-attention maps on prompts with complex sentence structure.}
To further highlight the generalizability of using the text self-attention matrix, we extend our analysis to more complex prompts, including those with \textit{relative pronouns} (\eg, who, which  etc.). Interestingly, Figure~\ref{fig:long_prompts} shows that the text self-attention matrix effectively captures the syntactic role of words like \textit{whose}, emphasizing the preceding vocabulary (e.g., \textit{valley}). In contrast to the text self-attention matrix, the text embedding similarity rarely exhibits this pattern.
In addition, Figure \ref{fig:TIFA_text_map} demonstrates the text self-attention maps for the MSCOCO captions included in TIFA benchmark.

\begin{figure}[ht]
    \centering
        \includegraphics[width=0.85\linewidth]{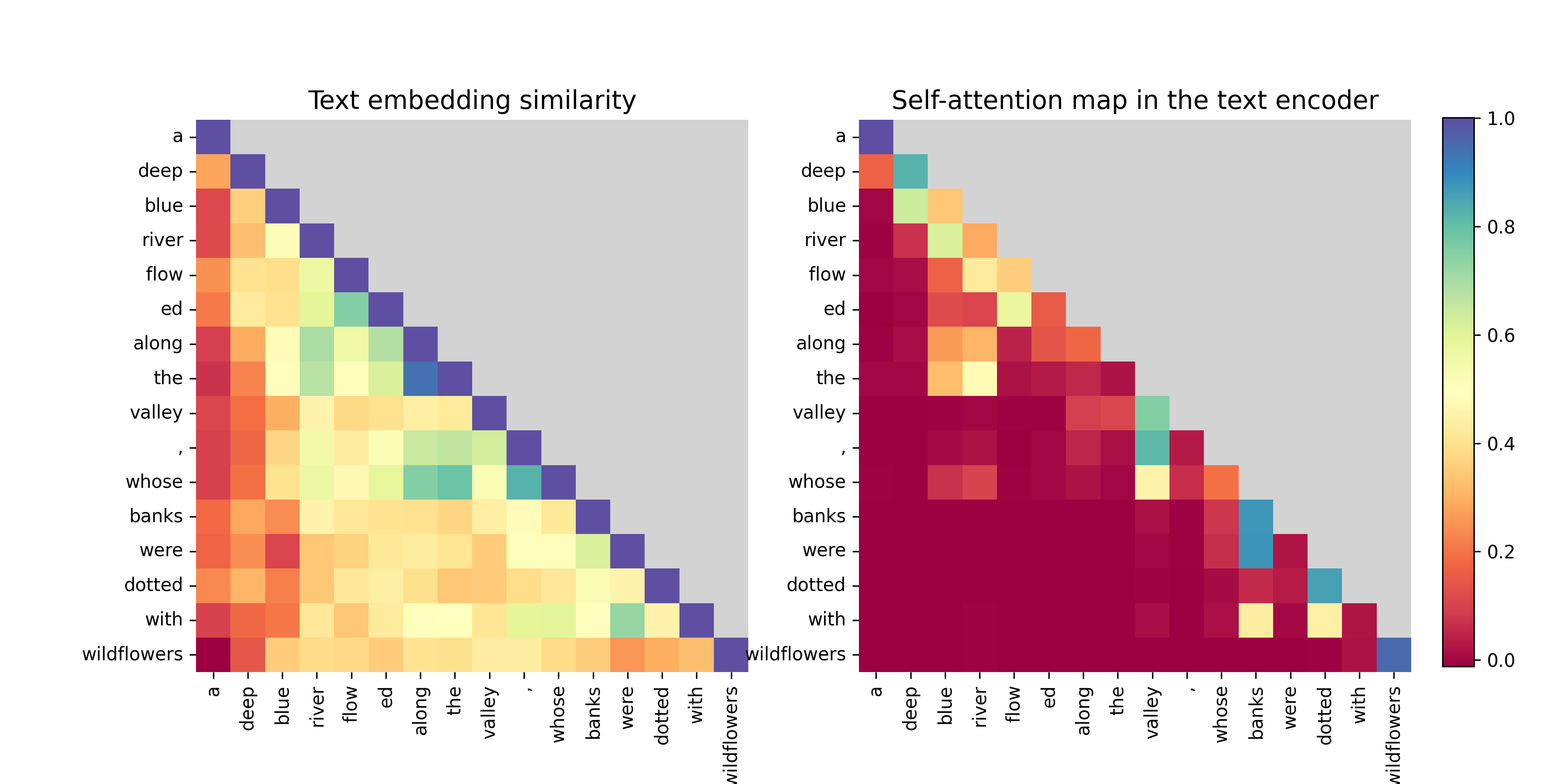} 
        \caption{Comparison of text embedding cosine similarity (left) and text self-attention maps (right) on the complex prompt, \textit{``A deep blue river flowed along the valley, whose banks were dotted with wildflowers''}. The prompt includes relative pronouns \textit{whose}.}
        \label{fig:long_prompts}
\end{figure}

\begin{figure*}[ht]
\centering
\begin{subfigure}[b]{0.45\textwidth}
    \includegraphics[width=\textwidth]{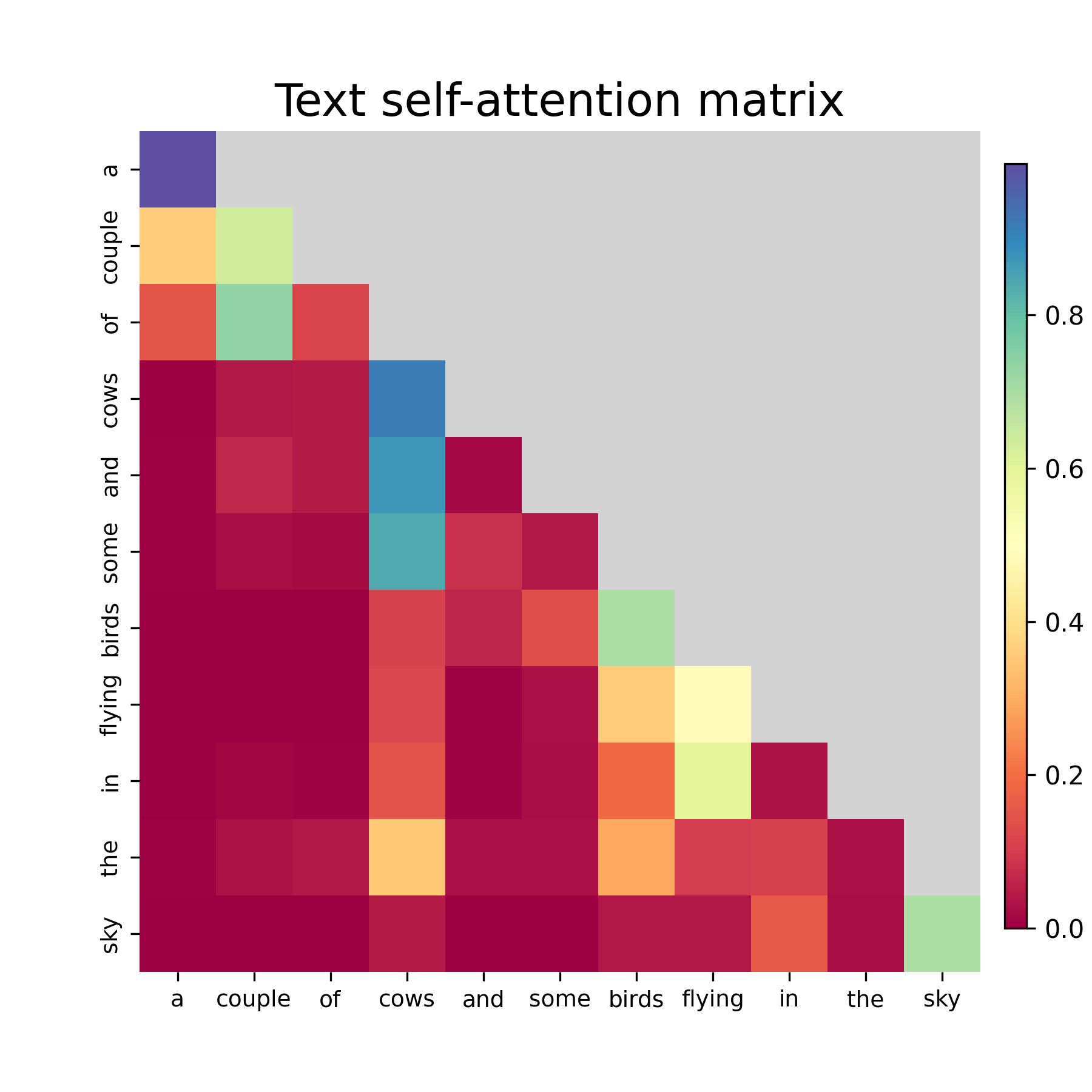}
    \caption{\textit{A couple of cows and some birds flying in the sky.}}
\end{subfigure}
\hfill
\begin{subfigure}[b]{0.45\textwidth}
    \includegraphics[width=\textwidth]{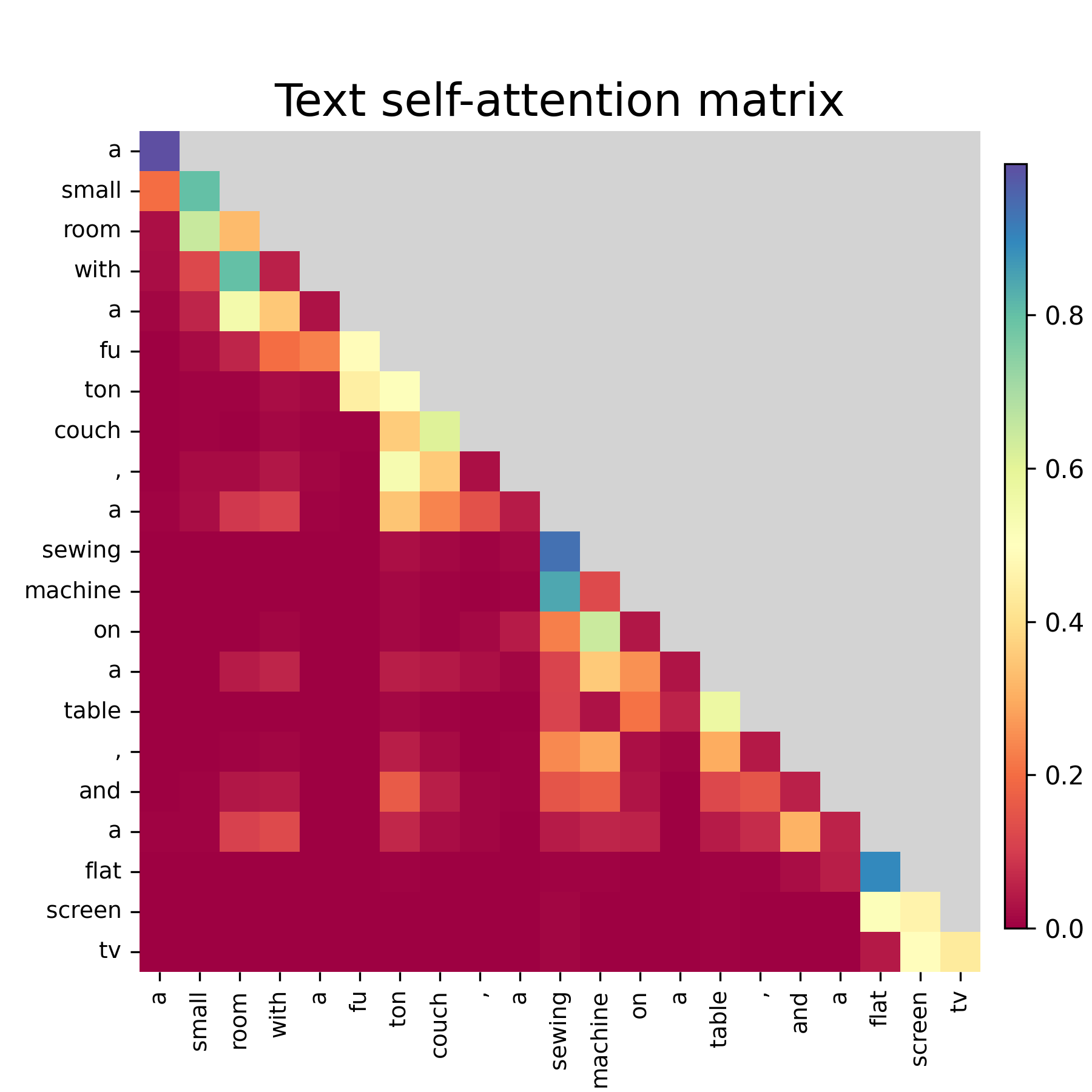}
    \caption{\textit{A small room with a futon couch, a sewing machine on a table, and a flatscreen TV.}}
\end{subfigure}

\vspace{1em}

\begin{subfigure}[b]{0.45\textwidth}
    \includegraphics[width=\textwidth]{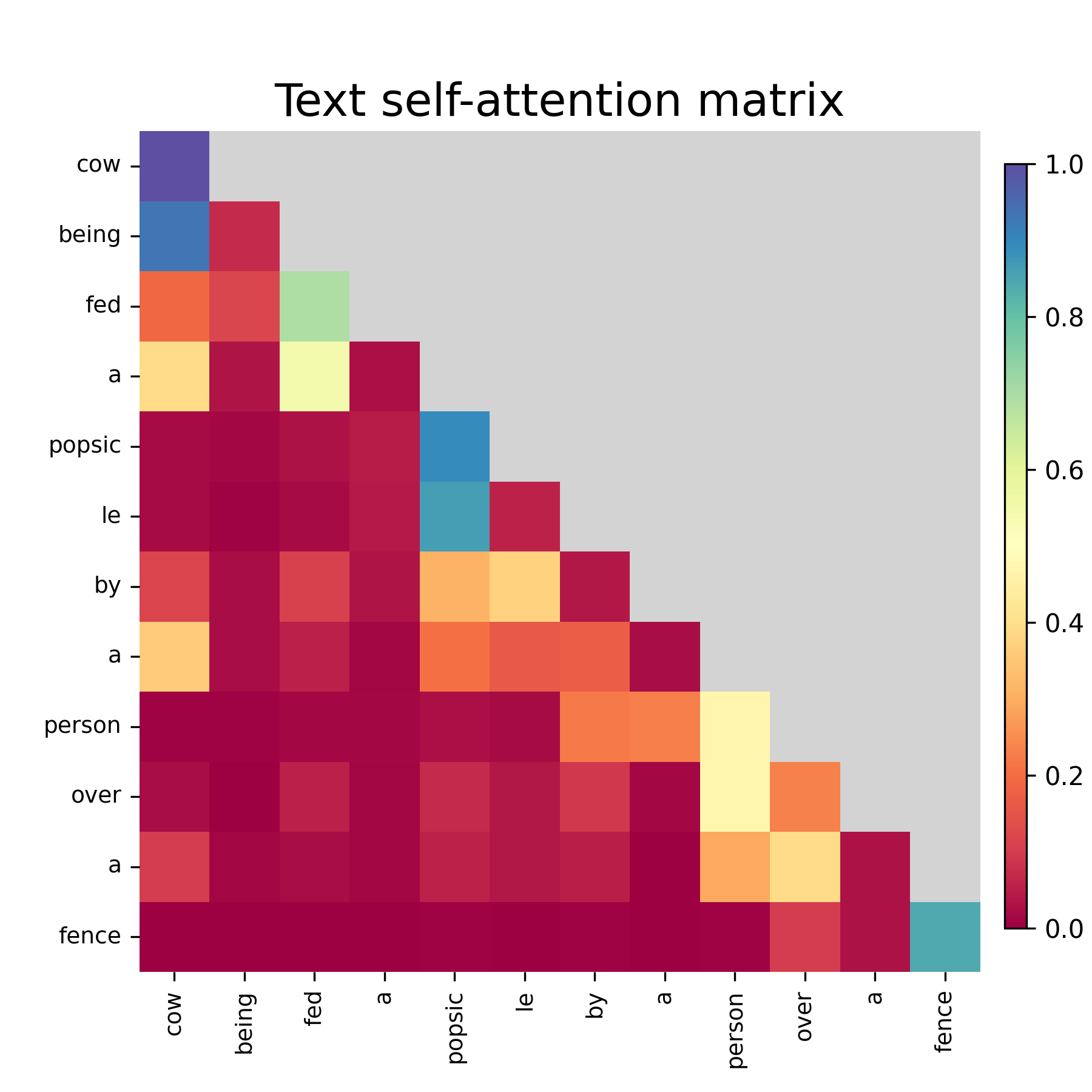}
    \caption{\textit{Cow being fed a popsicle by a person over a fence.}}
\end{subfigure}
\hfill
\begin{subfigure}[b]{0.45\textwidth}
    \includegraphics[width=\textwidth]{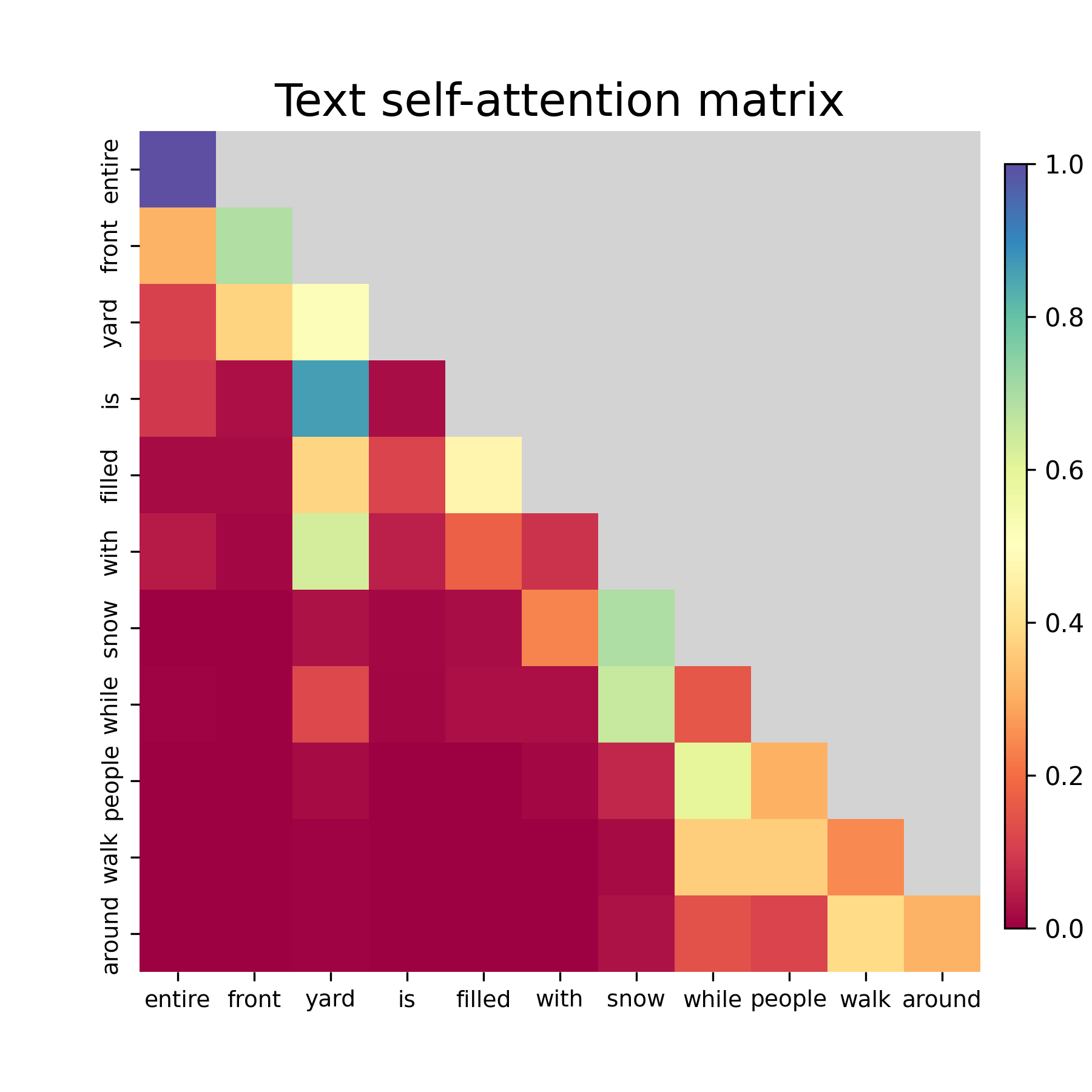}
    \caption{\textit{Entire front yard is filled with snow while people walk around.}}
\end{subfigure}
\caption{Text self-attention maps power by 3 for the MSCOCO captions included in TIFA benchmark.}
\label{fig:TIFA_text_map}
\end{figure*}

\paragraph{Experiment details in Section \ref{sec:exp}.}
Our method builds upon Stable Diffusion (SD) v1.5 \cite{rombach2022high}.
To enhance processing, we apply Gaussian smoothing to the cross-attention maps before computing cosine similarity, as discussed in \cite{chefer2023attend}. Additionally, we renormalize the text self-attention maps, excluding the \texttt{<BOS>} and \texttt{<EOS>} tokens as clarified in Eq.\eqref{eq:self_attn_renorm}. 
When computing loss function in Eq.\eqref{eq:loss_fn}, the first row, corresponding to the \texttt{<BOS>} token, is omitted from the computations. 
Regarding the prompt datasets, for \textit{Objects}, we use 66 prompts structured as ``$[\texttt{\footnotesize attribute}_1][\texttt{\footnotesize object}_1]$ \emph{and} $[\texttt{\footnotesize attribute}_2][\texttt{\footnotesize object}_2]$''. 
For \textit{Animals-Objects}, 144 prompts are employed. They are structured in two templates: ``$[\texttt{\footnotesize animal}]$ \emph{with} $[\texttt{\footnotesize object}]$'' and ``$[\texttt{\footnotesize animal}]$ \emph{and} $[\texttt{\footnotesize attribute}][\texttt{\footnotesize object}]$''. 

\paragraph{Ablation study.}
A larger $\alpha$ in our optimization process imposes stronger constraints on the latent variable ($z_t$), enhancing the regularization of cross-attention maps by aligning them more closely with the text self-attention maps. On the other hand, $\gamma$ serves as an exponent, amplifying larger values and compressing smaller ones, thereby controlling the temperature.
We conducted a grid search for $\alpha$ within the set $\{5, 10, 15, 25, 40\}$, and for $\gamma$, we explored values in $\{2, 3, 4\}$. The parameters that maximize the CLIP-full and CLIP-min similarity scores are chosen.
Based on the grid search, we ultimately selected a scale factor of $\alpha=10$ and $\gamma=4$ for \textit{Objects} and \textit{Animals-Objects}, and $\alpha=40$ and $\gamma=4$ for the TIFA benchmark, achieving the optimal balance between performance and output quality.

\end{document}